\documentclass[final,12pt]{styles/arxiv/colt2021}
\usepackage{graphicx,color} 
\usepackage[utf8]{inputenc} 
\usepackage[T1]{fontenc}    
\usepackage{hyperref}       
\usepackage{url}            
\usepackage{booktabs}       
\usepackage{amsfonts}       
\usepackage{nicefrac}       
\usepackage{microtype}      
\usepackage{xcolor}         
\usepackage{multirow,makecell}
\usepackage{enumitem}

\def \interior   {\text{int}}
\def \Rcal {\mathcal{R}}
\def \Ecal {\mathcal{E}}
\def \bX {\boldsymbol{\X}}

\def \bfh {\hat{\f}}
\def \bxi {\boldsymbol{\xi}}
\def \poly {\mbox{poly}}
\allowdisplaybreaks

\usepackage{xspace}
\usepackage{amsmath}
\usepackage{bbm}
\usepackage{algorithm}
\PassOptionsToPackage{algo2e, vlined, ruled, noend}{algorithm2e}
\usepackage{mathrsfs}
\usepackage{amssymb}
\usepackage{bm}
\usepackage{makecell}
\usepackage{tabulary}
\usepackage{xcolor}
\usepackage{prettyref}
\usepackage{mathtools}
\usepackage{amsmath}

\definecolor{Green}{rgb}{0.13, 0.65, 0.3}

\renewcommand{\tilde}{\widetilde}
\renewcommand{\hat}{\widehat}
\def \O {\mathcal{O}}
\def \Ot {\widetilde{\mathcal{O}}}
\def \X {\mathcal{X}}

\def \T {\top}

\def \x {\boldsymbol{x}}
\def \y {\boldsymbol{y}}

\def \h {\boldsymbol{h}}
\def \f {\boldsymbol{f}}

\newcommand{\LineComment}[1]{\hfill$\rhd\ $\text{#1}}

\def \R {\mathbb{R}}

\newcommand{\calB}{{\mathcal{B}}}
\newcommand{\calX}{{\mathcal{X}}}

\newcommand{\calU}{{\mathcal{U}}}

\newcommand{\calK}{{\mathcal{K}}}

\newcommand{\calT}{{\mathcal{T}}}

\newcommand{\Reg}{\textsc{Reg}}

\newcommand{\one}{\boldsymbol{1}}
\newcommand{\zero}{\boldsymbol{0}}

\newcommand{\dplus}[1]{\bm{#1}}

\newcommand{\dpg}{\dplus{g}}
\newcommand{\dpM}{\dplus{M}}
\newcommand{\dpQ}{\dplus{Q}}
\newcommand{\dpf}{\dplus{f}}
\newcommand{\dpX}{\dplus{\calX}}
\newcommand{\dpw}{\dplus{w}}
\newcommand{\dpF}{\dplus{F}}
\newcommand{\dpu}{\dplus{u}}
\newcommand{\dpb}{\dplus{b}}

\newcommand{\dpe}{\dplus{e}}
\newcommand{\dpx}{\dplus{x}}
\newcommand{\dpy}{\dplus{y}}
\newcommand{\dpH}{\dplus{H}}

\newcommand{\dpell}{\dplus{\ell}}

\newcommand{\dpR}{\dplus{R}}
\newcommand{\dpxi}{\dplus{\xi}}

\newcommand{\dpI}{\dplus{I}}

\newcommand{\dph}{\dplus{h}}

\DeclareMathOperator*{\argmin}{argmin}

\newcommand{\field}[1]{\mathbb{#1}}

\newcommand{\E}{\field{E}}

\newcommand{\inner}[1]{ \left\langle {#1} \right\rangle }

\newcommand{\norm}[1]{\left\|{#1}\right\|}

\usepackage{mathtools}
\DeclarePairedDelimiter\abs{\lvert}{\rvert}

\newcommand{\wh}{\widehat}
\newcommand{\wt}{\widetilde}

\newcommand{\ind}{\mathbbm{1}}

\newcommand{\order}{\ensuremath{\mathcal{O}}}
\newcommand{\otil}{\ensuremath{\tilde{\mathcal{O}}}}

\usepackage{nicefrac}

\usepackage{prettyref}
\newcommand{\pref}[1]{\prettyref{#1}}

\newcommand{\savehyperref}[2]{\texorpdfstring{\hyperref[#1]{#2}}{#2}}
\newrefformat{eq}{\savehyperref{#1}{Eq. \textup{(\ref*{#1})}}}
\newrefformat{eqn}{\savehyperref{#1}{Eq. \textup{(\ref*{#1})}}}
\newrefformat{lem}{\savehyperref{#1}{Lemma~\ref*{#1}}}
\newrefformat{lemma}{\savehyperref{#1}{Lemma~\ref*{#1}}}
\newrefformat{def}{\savehyperref{#1}{Definition~\ref*{#1}}}
\newrefformat{line}{\savehyperref{#1}{Line~\ref*{#1}}}
\newrefformat{thm}{\savehyperref{#1}{Theorem~\ref*{#1}}}
\newrefformat{theorem}{\savehyperref{#1}{Theorem~\ref*{#1}}}
\newrefformat{corr}{\savehyperref{#1}{Corollary~\ref*{#1}}}
\newrefformat{cor}{\savehyperref{#1}{Corollary~\ref*{#1}}}
\newrefformat{col}{\savehyperref{#1}{Corollary~\ref*{#1}}}
\newrefformat{sec}{\savehyperref{#1}{Section~\ref*{#1}}}
\newrefformat{subsec}{\savehyperref{#1}{Section~\ref*{#1}}}
\newrefformat{appendix}{\savehyperref{#1}{Appendix~\ref*{#1}}}
\newrefformat{app}{\savehyperref{#1}{Appendix~\ref*{#1}}}

\newrefformat{assum}{\savehyperref{#1}{Definition~\ref*{#1}}}

\newrefformat{assumption}{\savehyperref{#1}{Definition~\ref*{#1}}}
\newrefformat{ex}{\savehyperref{#1}{Example~\ref*{#1}}}
\newrefformat{fig}{\savehyperref{#1}{Figure~\ref*{#1}}}
\newrefformat{alg}{\savehyperref{#1}{Algorithm~\ref*{#1}}}
\newrefformat{algo}{\savehyperref{#1}{Algorithm~\ref*{#1}}}
\newrefformat{rem}{\savehyperref{#1}{Remark~\ref*{#1}}}
\newrefformat{conj}{\savehyperref{#1}{Conjecture~\ref*{#1}}}
\newrefformat{prop}{\savehyperref{#1}{Proposition~\ref*{#1}}}
\newrefformat{proto}{\savehyperref{#1}{Protocol~\ref*{#1}}}
\newrefformat{prob}{\savehyperref{#1}{Problem~\ref*{#1}}}
\newrefformat{claim}{\savehyperref{#1}{Claim~\ref*{#1}}}
\newrefformat{que}{\savehyperref{#1}{Question~\ref*{#1}}}
\newrefformat{op}{\savehyperref{#1}{Open Problem~\ref*{#1}}}
\newrefformat{fn}{\savehyperref{#1}{Footnote~\ref*{#1}}}
\newrefformat{tab}{\savehyperref{#1}{Table~\ref*{#1}}}
\newrefformat{fig}{\savehyperref{#1}{Figure~\ref*{#1}}}

\title[Adaptive Bandit Convex Optimization with Heterogeneous Curvature]{Adaptive Bandit Convex Optimization with Heterogeneous Curvature}
\usepackage{times}

\coltauthor{%
	\Name{Haipeng Luo\nametag{\thanks{Authors are listed in alphabetical order.}}} \Email{haipengl@usc.edu}\\
	\addr University of Southern California
	\AND
	\Name{Mengxiao Zhang\nametag{\footnotemark[1]}} \Email{mengxiao.zhang@usc.edu}\\
	\addr University of Southern California
	\AND
	\Name{Peng Zhao\nametag{\footnotemark[1]}} \Email{zhaop@lamda.nju.edu.cn}\\
	\addr National Key Laboratory for Novel Software Technology, Nanjing University
}

\begin{document}

\maketitle

\begin{abstract}%
We consider the problem of adversarial bandit convex optimization, that is, online learning over a sequence of arbitrary convex loss functions with only one function evaluation for each of them. While all previous works assume known and homogeneous curvature on these loss functions, we study a heterogeneous setting where each function has its own curvature that is only revealed after the learner makes a decision.
We develop an efficient algorithm that is able to adapt to the curvature on the fly. Specifically, our algorithm not only recovers or \emph{even improves} existing results for several homogeneous settings, but also leads to surprising results for some heterogeneous settings --- for example, while~\citet{conf/nips/HazanL14} showed that $\otil(d^{\nicefrac{3}{2}}\sqrt{T})$ regret is achievable for a sequence of $T$ smooth and strongly convex $d$-dimensional functions,
our algorithm reveals that the same is achievable even if $T^{\nicefrac{3}{4}}$ of them are not strongly convex, and sometimes even if a constant fraction of them are not strongly convex.
Our approach is inspired by the framework of~\citet{NIPS'07:AOGD} who studied a similar heterogeneous setting but with stronger gradient feedback.
Extending their framework to the bandit feedback setting requires novel ideas such as lifting the feasible domain and using a logarithmically homogeneous self-concordant barrier regularizer.
\end{abstract}

\section{Introduction}
\label{sec:intro}

We consider the problem of adversarial bandit convex optimization, formulated as the following sequential learning process of $T$ rounds. At the beginning, knowing the learner's algorithm, an adversary decides an arbitrary sequence of $T$ convex loss functions $f_1,\dots,f_T:\X\mapsto \R$ over some convex domain $\X \subset \R^d$. Then, at each round $t$, the learner is required to select a point $x_t \in \X$, and afterwards observes and suffers her loss $f_t(x_t)$.
The performance of the learner is measured by her regret, the difference between the her total suffered loss and that of the best fixed point in hindsight.

Without further assumption, the best existing result is from~\citep{JACM'21:kernel-BCO} which achieves $\otil(d^{10.5}\sqrt{T})$ regret with large computational complexity of $\otil(\poly(d)T)$ per round.
On the other hand, the current best lower bound is $\Omega(d\sqrt{T})$~\citep{dani2007price}, exhibiting a large gap in the $d$ dependency.
It has been shown that, however, curvature of the loss functions helps --- for example, when the functions are all smooth and strongly convex, \citet{conf/nips/HazanL14} develop a simple and efficient Follow-the-Regularized-Leader (FTRL) type algorithm with $\otil(d^{\nicefrac{3}{2}}\sqrt{T})$ regret;
even when the functions are only smooth, \citet{AISTATS'11:smooth-BCO} show that $\otil(dT^{\nicefrac{2}{3}})$ regret is achievable again via a simple and efficient FTRL variant, despite the suboptimal dependency in $T$.

However, all such existing results making use of curvature assume a homogeneous setting, that is, all loss functions share the same curvature parameters that are known ahead of time.
Ignoring the ubiquitous heterogeneity in online data is either unrealistic or forcing one to use a conservative curvature parameter (e.g., the smallest strong convexity parameter among all functions),
while intuitively, being able to exploit and adapt to the individual curvature of each loss function should result in much better performance.

\begin{table}[!t]
\centering
\caption{A summary of our results for bandit convex optimization over $T$ smooth $d$-dimensional functions, the $t$-th of which is $\sigma_t$-strongly convex.
$\calT \subset [T]$ is a subset of rounds with no strong convexity. 
The dependency on parameters other than $d$ and $T$ can be found in the respective corollary (see also \pref{fn:discontinuity}).
Note that our results are all achieved by one single adaptive algorithm.
}
\label{tab: res_smooth}
\renewcommand*{\arraystretch}{1.4}
\resizebox{\textwidth}{!}{
\begin{tabular}{|l|c|c|}
\hline
\multicolumn{1}{|c|}{\textbf{Strong Convexity} $\{\sigma_t\}_{t=1}^T$} &  \multicolumn{1}{c|}{\textbf{Previous Works}} & \multicolumn{1}{c|}{\textbf{Our Results} (\pref{alg:lift-smooth})} \\\hline
$\sigma_t=0$, $\forall t\in [T]$        & $\otil(dT^{\nicefrac{2}{3}})$ \citep{AISTATS'11:smooth-BCO}    & $\otil(d^{\nicefrac{2}{3}}T^{\nicefrac{2}{3}})$ (\pref{cor: only-smooth})    \\ \hline
$\sigma_t=\sigma>0$, $\forall t\in [T]$ & $\otil(d^{\nicefrac{3}{2}}\sqrt{T})$ \citep{conf/nips/HazanL14} & $\otil(d^{\nicefrac{3}{2}}\sqrt{T})$ (\pref{cor: strongly-convex-smooth}) \\ \hline
\makecell[l]{$\sigma_t= \sigma\ind\{t \notin \calT\}$, \\
$|\calT|=T^{\nicefrac{3}{4}}$ or $\calT = [\nicefrac{T}{2}, T]$}
&
   \multicolumn{1}{c|}{N/A} &
  $\otil(d^{\nicefrac{3}{2}}\sqrt{T})$ (\pref{cor: worst-best-case}) \\ \hline
$\sigma_t=t^{-\alpha}$,  $\forall t\in [T]$ &
  \multicolumn{1}{c|}{N/A} &
  $\begin{cases}\otil(d^{\nicefrac{3}{2}}T^{\nicefrac{(1+\alpha)}{2}}),&\alpha\in[0,\nicefrac{1}{3})\\\otil(d^{\nicefrac{2}{3}}T^{\nicefrac{2}{3}}),&\alpha\in[\nicefrac{1}{3},1]\end{cases}$ (\pref{cor: intermediate-case})\\ \hline
\end{tabular}}
\end{table}

Motivated by this fact and inspired by the work of~\citet{NIPS'07:AOGD} who consider a heterogeneous setting in Online Convex Optimization with stronger gradient feedback,
we study a similar setting where each $f_t$ has its own strong convexity $\sigma_t\geq 0$, revealed only at the end of round $t$ after the learner decides $x_t$.
We provide examples in \pref{sec: pre} to illustrate why this is a realistic setup even in the bandit setting where traditionally only $f_t(x_t)$ is revealed.
In this setting, we develop efficient algorithms that automatically adapt to the heterogeneous curvature and enjoy strong adaptive regret bounds.
These bounds not only recover or \emph{even improve} the existing results in the homogeneous setting, but also reveal interesting new findings in some hybrid scenarios.
More specifically, our results are as follows (for simplicity, only the dependency on $d$ and $T$ is shown; see respective sections for the complete bounds).

\begin{itemize}[leftmargin=*]
  \setlength\itemsep{0em}

\item We start with the case where all loss functions are smooth in \pref{sec:smooth-BCO}. 
Our algorithm achieves $\otil(d^{\nicefrac{2}{3}}T^{\nicefrac{2}{3}})$ regret if no functions are strongly convex (that is, $\sigma_t = 0$ for all $t$), \emph{improving} the $\otil(dT^{\nicefrac{2}{3}})$ bound of~\citep{AISTATS'11:smooth-BCO},
and $\otil(d^{\nicefrac{3}{2}}\sqrt{T})$ regret if all functions happen to be $\sigma$-strongly convex (that is, $\sigma_t = \sigma$ for all $t$), matching that from~\citep{conf/nips/HazanL14}.
In fact, our algorithm achieves the latter result \emph{even if $\Theta(T^{\nicefrac{3}{4}})$ of the functions have no strong convexity} (and sometimes even if \emph{a constant fraction of the functions} have no strong convexity).
More generally, our bound interpolates between these two extremes.
For example, if $\sigma_t$ is decaying at the rate of $\nicefrac{1}{t^\alpha}$ for some $\alpha \in [0,1]$, then the regret is $\otil(d^{\nicefrac{3}{2}}T^{\nicefrac{(1+\alpha)}{2}})$ when $\alpha \leq \nicefrac{1}{3}$, and $\otil(d^{\nicefrac{2}{3}}T^{\nicefrac{2}{3}})$ otherwise.\footnote{This is a simplified and loosen version of \pref{cor: intermediate-case}, which explains the discontinuity in $\alpha$. \label{fn:discontinuity}
}
See \pref{tab: res_smooth} for a summary.
We note that the improvement over~\citep{AISTATS'11:smooth-BCO} comes as a side product of the better regularization technique of our algorithm.

\item We then consider another scenario where all loss functions are Lipschitz in \pref{sec:lipschitz-BCO}.
We develop another algorithm that achieves $\otil(\sqrt{d}T^{\nicefrac{3}{4}})$ regret when no functions are strongly convex, improving the $\otil(d^{\nicefrac{3}{4}}T^{\nicefrac{3}{4}})$ bound of~\citep{lecture18},
and $\otil(d^{\nicefrac{2}{3}}T^{\nicefrac{2}{3}})$ regret when all functions are $\sigma$-strongly convex, improving the $\otil(d^{\nicefrac{4}{3}}T^{\nicefrac{2}{3}})$ bound of~\citep{COLT'10:agarwal_optimal}.
These improvements again come as a side product of our better regularization.
Similarly, the $\otil(d^{\nicefrac{2}{3}}T^{\nicefrac{2}{3}})$ result holds even if $\Theta(T^{\nicefrac{8}{9}})$ of the functions (or sometimes a constant fraction of them) have no strong convexity.
For similar intermediate bounds in the example when $\sigma_t = \nicefrac{1}{t^\alpha}$, 
see \pref{tab: res_lipschitz}.
\end{itemize}

\begin{table}[!t]
\centering
\caption{A summary of our results for bandit convex optimization over $T$ Lipschitz $d$-dimensional functions, the $t$-th of which is $\sigma_t$-strongly convex.
$\calT \subset [T]$ is a subset of rounds with no strong convexity. 
The dependency on parameters other than $d$ and $T$ can be found in the respective corollary.
Note that our results are all achieved by one single adaptive algorithm.
}
\label{tab: res_lipschitz}
\renewcommand*{\arraystretch}{1.4}
\resizebox{\textwidth}{!}{
\begin{tabular}{|l|c|c|}
\hline
\multicolumn{1}{|c|}{\textbf{Strong Convexity} $\{\sigma_t\}_{t=1}^T$} &  \multicolumn{1}{c|}{\textbf{Previous Works}} & \multicolumn{1}{c|}{\textbf{Our Results} (\pref{alg:lift-Lipschitz})} \\\hline
$\sigma_t=0$, $\forall t\in [T]$        & $\otil(d^{\nicefrac{3}{4}}T^{\nicefrac{3}{4}})$ \citep{lecture18}    & $\otil(\sqrt{d}T^{\nicefrac{3}{4}})$ (\pref{cor: only-Lipschitz})    \\ \hline
$\sigma_t=\sigma>0$, $\forall t\in [T]$ & $\otil(d^{\nicefrac{4}{3}}T^{\nicefrac{2}{3}})$ \citep{COLT'10:agarwal_optimal} & $\otil(d^{\nicefrac{2}{3}}T^{\nicefrac{2}{3}})$ (\pref{cor: only-Lipschitz-strongly-cvx}) \\ \hline
\makecell[l]{$\sigma_t= \sigma\ind\{t \notin \calT\}$, \\
$|\calT|=T^{\nicefrac{8}{9}}$ or $\calT = [\nicefrac{T}{2}, T]$} &
  \multicolumn{1}{c|}{N/A} &
  $\otil(d^{\nicefrac{2}{3}}T^{\nicefrac{2}{3}})$ (\pref{cor: only-lipschitz-worst-best-case}) \\ \hline
$\sigma_t=t^{-\alpha}$,  $\forall t\in [T]$ &
  \multicolumn{1}{c|}{N/A} &
  $\begin{cases}\otil(d^{\nicefrac{2}{3}}T^{\nicefrac{(2+\alpha)}{3}}),&\alpha\in[0,\nicefrac{1}{4})\\\otil(\sqrt{d}T^{\nicefrac{3}{4}}),&\alpha\in[\nicefrac{1}{4},1]\end{cases}$ (\pref{cor: only-lipschitz-intermediate-case})\\ \hline
\end{tabular}}
\end{table}

\paragraph{Techniques.} 
Our algorithm is also a variant of FTRL, with two crucial new ingredients to handle heterogeneous curvature.
First, we extend the idea of~\citep{NIPS'07:AOGD} to adaptively add $\ell_2$ regularization to the loss functions and adaptively tune the learning rate.
Doing so in the bandit setting is highly nontrivial and requires our second technical ingredient, which is to lift the problem to the $(d+1)$-dimensional space and then apply a \emph{logarithmically homogeneous self-concordant barrier} in the FTRL update. 
This technique is inspired by a recent work of~\citet{NIPS'20:unbiased-bandits} on 
achieving high-probability regret bounds for adversarial linear bandits,
but the extension from linear bandits to convex bandits is nontrivial.
In fact, the purpose of using this technique is also different: they need to bound the variance of the learner's loss, which is related to bounding $x^\top\nabla^2\psi(x)x$ for some regularizer $\psi$, while we need to bound the stability of the algorithm, which is related to bounding $\nabla\psi(x)\nabla^{-2}\psi(x) \nabla\psi(x)$, but it turns out that when $\psi$ is a logarithmically homogeneous $\nu$-self-concordant barrier, then these two quantities are \emph{exactly the same} and bounded by $\nu$.

\paragraph{Related work.} Bandit convex optimization has been extensively studied under different loss function structures, including Lipschitz functions~\citep{conf/nips/Kleinberg04,SODA'05:Flaxman-BCO}, linear functions~\citep{Competing:Dark, abernethy2012interior,bubeck2012towards}, smooth functions~\citep{AISTATS'11:smooth-BCO}, strongly convex functions~\citep{COLT'10:agarwal_optimal}, smooth and strongly convex functions~\citep{conf/nips/HazanL14,ito2020optimal},
quadratic functions~\citep{COLT'21:efficient-BCO}, pseudo-1-dimensional functions~\citep{ICML'21:BCO-1d}, and others.
Without any structure (other than convexity), a series of progress has been made over recent years~\citep{bubeck2015bandit, hazan2016optimal, COLT'16:multi-scale-BCO,JACM'21:kernel-BCO}, 
but as mentioned, even the best result~\citep{JACM'21:kernel-BCO} has a large dependency on $d$ in the regret and is achieved by an impractical algorithm with large computational complexity. 
Our comparisons in this work (such as those in \pref{tab: res_smooth} and \pref{tab: res_lipschitz}) thus mainly focus on more efficient and practical methods in the literature that share the same FTRL framework.\footnote{When the functions are smooth only, 
our comparison is based on~\citep{AISTATS'11:smooth-BCO}, instead of the seemingly better results of~\citep{dekel2015bandit, yang2016optimistic}, because the latter ones are unfortunately wrong as pointed out in~\citep{hu2016bandit}.
}

Closest to our heterogeneous setting is the work on Online Convex Optimization by~\citet{NIPS'07:AOGD}, where at the end of each round, $\sigma_t$ and $\nabla f_t(x_t)$ are revealed (versus $\sigma_t$ and $f_t(x_t)$ in our setting).
Due to the stronger feedback, their algorithm achieves $\order(\sqrt{T})$ regret without any strong convexity, $\order(\log T)$ regret if all functions are strongly convex, and generally something in between.
Our results are in the same vein, and as mentioned, our algorithm is also heavily inspired by theirs.

Another potential approach to adapting to different environments is to have a meta algorithm learning over a set of base algorithms, each dedicated to a specific environment.
Doing so in the bandit setting, however, is highly challenging~\citep{agarwal2017corralling} or even impossible sometimes~\citep{marinov2021pareto}.
For example, even if one only aims to adapt to two environments, one with only smooth functions and the other with smooth and strongly convex functions, the approach of~\citep{agarwal2017corralling} is only able to achieve $\otil(T^{\nicefrac{3}{4}})$ regret for the first environment if one insists to enjoy $\otil(\sqrt{T})$ regret in the second one.

\section{Preliminaries and Problem Setup}\label{sec: pre}

We start by reviewing some basic definitions.
\begin{definition}
\label{assumption:smoothness}
We say that a differentiable function $f: \X \mapsto \R$ is $\beta$-smooth over the feasible set $\X$ if for any $x, y \in \X$, $\norm{\nabla f(x)-\nabla f(y)}_2 \leq \beta \norm{x-y}_2$ holds.
\end{definition}

\begin{definition}
\label{assumption:lipschitz}
We say that a function $f: \X \mapsto \R$ is $L$-Lipschitz over the feasible set $\X$ if for any $x, y \in \X$, $\abs{f(x) - f(y)} \leq L \norm{x-y}_2$ holds.
\end{definition}

\begin{definition}
\label{assumption:strongly-convex}
We say that a differentiable function $f: \X \mapsto \R$ is $\sigma$-strongly convex over the feasible set $\X$ if for any $x, y \in \X$, $f(y)\geq f(x) + \nabla f(x)^\top(y-x) + \frac{\sigma}{2}\|x-y\|_2^2$ holds.
\end{definition}

\paragraph{Problem setup.}
Bandit Convex Optimization (BCO) can be modeled as a $T$-round games between a learner and an oblivious adversary.
Before the game starts, the adversary (knowing the learner's algorithm) secretly decides an arbitrary sequence of convex functions $f_1,\dots,f_T:\X\mapsto \R$, where $\calX \subset \R^d$ is a known compact convex domain. 
In \pref{sec:smooth-BCO} we assume that all $f_t$'s are $\beta$-smooth for some known parameter $\beta \geq 0$,
while in \pref{sec:lipschitz-BCO} we assume that they are all $L$-Lipschitz for some known parameter $L \geq 0$.
In both cases, we denote by $\sigma_t \geq 0$ the strong convexity parameter of $f_t$, initially \emph{unknown} to the learner (note that $\sigma_t$ could be zero).

At each round $t \in [T] \triangleq\{1, \ldots, T\}$ of the game, the learner chooses an action $x_t\in \X$ based on all previous observations, and subsequently suffers loss $f_t(x_t)$. 
The adversary then reveals both $f_t(x_t)$ and $\sigma_t$ to the learner.
Compared to previous works where $\sigma_t$ is the same for all $t$ and known to the learner ahead of time, our setting is clearly more suitable for applications with heterogeneous curvature.  
We provide such an example below, which also illustrates why it is reasonable for the learner to observe $\sigma_t$ at the end of round $t$.

\paragraph{Examples.}
Consider a problem where $f_t(x)=\sum_{i=1}^N g_{t,i}(c_{t,i}^\top x)$.
Here, $N$ is the number of users in each round, $c_{t,i}$ represents some context of the $i$-th user in round $t$, the learner's decision $x$ is used to make a linear prediction $c_{t,i}^\top x$ for this user, and her loss is evaluated via a convex function $g_{t,i}$ which incorporates some ground truth for this user (e.g., labels in the case of classification where $g_{t,i}$ could be the logistic loss, or responses in the case of regression where $g_{t,i}$ could be the squared loss).
Note that the (heterogeneous) strong convexity $\sigma_t$ of $f_t$ is at least $\mu \lambda_{\min}(\sum_{i=1}^N c_{t,i}c_{t,i}^\top)$ where $\lambda_{\min}(\cdot)$ denotes the minimum eigenvalue of a matrix, and $\mu$ is a lower bound on the second derivative of $g_{t,i}$, often known ahead of time assuming some natural boundedness of $x$ and $c_{t,i}$.

In this setup, there are several situations where our feedback model is reasonable.
For example, it might be the case that the loss $f_t(x_t)$ as well as the context $c_{t,i}$ is visible to the learner, but the ground truth (and thus $g_{t,i}$) is not.
In this case, based on the earlier bound on the strong convexity, the learner can calculate  $\sigma_t$ herself.
As another example, due to privacy consideration, the user's context $c_{t,i}$ might not be revealed to the learner, but it is acceptable to reveal a single number $\lambda_{\min}(\sum_{i=1}^N c_{t,i}c_{t,i}^\top)$ summarizing this batch of users. 
Clearly, $\sigma_t$ can also be calculated by the learner in this case.

\paragraph{Objective and simplifying assumptions.}
The objective of the learner is to minimize her (expected) \emph{regret}, defined as
\begin{equation}
  \label{eq:regret}
  \Reg = \mathbb{E} \left[\sum_{t=1}^{T} f_t(x_t)\right] - \min_{x\in\calX}\mathbb{E}\left[\sum_{t=1}^{T} f_t(x)\right],
\end{equation}
which is the difference between the expected loss suffered by the learner and that of the best fixed action (the expectation is with respect to the randomness of the learner). Without loss of generality, we assume $\max_{x\in \calX }|f_t(x)|\leq 1$ for all $t\in [T]$, $\calX$ contains the origin, and $\max_{x\in \calX}\|x\|_2=1$.\footnote{This is without loss of generality because for any problem with $|f_t(x)|\leq B$ and $\max_{x,x'\in \calX}\|x-x'\|_2=D$, we can solve it via solving a modified problem with convex domain $\calU = \{u = \frac{1}{D}(x-x_0): x \in \calX\}$ for any fixed $x_0 \in \calX$ and loss functions $g_t(u) = \frac{1}{B}f_t(Du+x_0)$, which then satisfies our simplifying assumptions $\max_{u\in \calU }|g_t(u)|\leq 1$, $\zero\in\calU$, and $\max_{u\in \calU}\|u\|_2=1$.
}

\paragraph{Notations.} We adopt the following notational convention throughout the paper. Generally, we use lowercase letters to denote vectors and capitalized letters to denote matrices. For a positive semi-definite matrix $M \in \R^{d \times d}$ and a vector $x \in \R^d$, $\|x\|_2$ represents the standard Euclidean norm of $x$ and $\norm{x}_M \triangleq \sqrt{x^\T M x}$ represents the Mahalanobis norm induced by $M$. $\mathbb{S}^d$ and $\mathbb{B}^d$ respectively denote the unit sphere and unit ball in $\R^d$, $\one$ and $\zero$ respectively denote the all-one and all-zero vectors with an appropriate dimension, and $I$ denotes the identity matrix with an appropriate dimension. 
For a differentiable convex function $\psi$, define the corresponding Bregman divergence as $D_{\psi}(u,v)=\psi(u)-\psi(v)-\inner{\nabla\psi(v), u-v}$. For a vector $x \in \R^d$, we use $x_{[i:j]}\in \R^{j-i+1}$ to denote the induced truncated vector consisting of the $i$-th to $j$-th coordinates of $x$. For a sequence of scalars $a_1,\ldots,a_t$, we use $a_{i:j} \triangleq \sum_{k=i}^j a_k$ to denote the cumulative summation, and $\{a_s\}_{s=p}^q$ to denote the subsequence $a_p,a_{p+1},\ldots,a_q$. The $\tilde{\O}(\cdot)$ notation omits the logarithmic dependence on the horizon $T$.\footnote{In the texts, for simplicity $\otil(\cdot)$ might also hide dependency on other parameters such as the dimension $d$. However, in all formal theorem/lemma statements, this will not be the case.} $\E_t[\cdot]$ is a shorthand for the conditional expectation given the history before round $t$. 
\section{Smooth Bandit Convex Optimization with Heterogeneous Strong Convexity}\label{sec:smooth-BCO}
Throughout this section, we assume that all loss functions $f_1, \ldots, f_T$ are $\beta$-smooth (see~\pref{assumption:smoothness}).
To present our adaptive algorithm in this case,
we start by reviewing two important existing algorithms upon which ours is built.

\paragraph{Review of Adaptive Online Gradient Descent (AOGD).} 
As mentioned, the AOGD algorithm of~\citep{NIPS'07:AOGD} is designed for a similar heterogeneous setting but with the stronger gradient feedback.
The first key idea of AOGD is that, instead of learning over the original loss functions $\{f_t\}_{t=1}^T$, one should learn over their $\ell_2$-regularized version $\{\wt{f}_t\}_{t=1}^T$ where $\wt{f}_t(x) = f_t(x)  + \frac{\lambda_t}{2} \norm{x}_2^2$ for some coefficient $\lambda_t$.
Intuitively, $\lambda_t$ is large when previous loss functions exhibit not enough strong convexity to stabilize the algorithm, and small otherwise.
How to exactly tune $\lambda_t$ based on the observed $\{\sigma_s\}_{s=1}^{t}$ is their second key idea --- they show that $\lambda_t$ should balance two terms and satisfy
\begin{equation}\label{eq:AOGD_tuning}
\frac{3}{2}\lambda_t = \frac{1}{\sigma_{1:t} + \lambda_{1:t}},
\end{equation}
which results in a quadratic equation of $\lambda_t$ and can be solved in closed form.
The final component of AOGD is simply to run gradient descent on $\{\wt{f}_t\}_{t=1}^T$ with some adaptively decreasing learning rates.

\paragraph{Review of BCO with smoothness and strong convexity.}
\citet{conf/nips/HazanL14} consider the BCO problem with $\beta$-smooth and $\sigma$-strongly convex loss functions.
Their FTRL-based algorithm maintains an auxiliary sequence $y_1, \ldots, y_T \in \X$ via
\begin{equation}
    \label{eq:FTRL-strongly-convex}
    y_{t+1} = \argmin_{x \in \X} \Bigg\{\sum_{s=1}^t \Big(\inner{g_s, x} + \frac{\sigma}{2} \norm{x - y_{s}}_2^2\Big) + \frac{1}{\eta}\psi(x)\Bigg\},
\end{equation}
where $g_s$ is some estimator of $\nabla f_s(y_s)$, $\eta > 0$ is some fixed learning rate, and the regularizer $\psi$ is a $\nu$-self-concordant barrier whose usage in BCO is pioneered by~\citet{Competing:Dark} (see \pref{appendix: self-concordant} for definition).
The rational behind the squared distance terms in this update is that, due to strong convexity, we have $f_s(y_s) - f_s(x) \leq \wt{g}_s(y_s) - \wt{g}_s(x)$ for any $x$ and $\wt{g}_s(x) = \nabla f_s(y_s)^\top x + \frac{\sigma}{2} \norm{x - y_{s}}_2^2$,
meaning that it suffices to consider $\wt{g}_1, \ldots, \wt{g}_T$ as the loss functions.
Having $y_t$, the algorithm makes the final prediction $x_t$ by adding certain curvature-adaptive and shrinking exploration to $y_t$: $x_t = y_t + H_t^{-\nicefrac{1}{2}} u_t$, where $H_t = \nabla^2 \psi(y_t) +\eta\sigma tI $ and $u_t$ is chosen from the unit sphere $\mathbb{S}^d$ uniformly at random ($x_t \in \calX$ is guaranteed by the property of self-concordant barriers).
Finally, with the feedback $f_t(x_t)$, the gradient estimator is constructed as $g_t=d\cdot f_t(x_t)H_t^{\nicefrac{1}{2}}u_t$,
which can be shown to be an unbiased and low-variance estimator of the gradient of some smoothed version of $f_t$ at $y_t$.

\subsection{Proposed Algorithm and Main Theorem}
We are now ready to describe our algorithm. 
Following~\citep{NIPS'07:AOGD}, our first step is also to consider learning over the $\ell_2$-regularized loss functions: $\wt{f}_t(x)=f_t(x)+\frac{\lambda_t}{2}\|x\|_2^2$ with an adaptively chosen $\lambda_t>0$ (note that with the bandit feedback $f_t(x_t)$, we can also evaluate $\wt{f}_t(x_t)$).
While~\citet{NIPS'07:AOGD} apply gradient descent, the standard and optimal algorithm for strongly-convex losses with gradient feedback, 
here we naturally apply the algorithm of~\citep{conf/nips/HazanL14} to this sequence of regularized loss functions instead.
Since $\wt{f}_t$ is $(\sigma_t + \lambda_t)$-strongly convex,
following \pref{eq:FTRL-strongly-convex} and adopting a decreasing learning rate $\eta_t$
shows that we should maintain the auxiliary sequence $y_1, \ldots, y_T$ according to
\begin{equation}
    \label{eq:FTRL-no-lifting}
    y_{t+1} = \argmin_{x \in \X} \Bigg\{\sum_{s=1}^t \Big(\inner{g_s, x} + \frac{\sigma_s+\lambda_s}{2} \norm{x - y_{s}}_2^2\Big) + \frac{1}{\eta_{t+1}}\psi(x)\Bigg\},
\end{equation}
where similarly $g_t=d\cdot \wt{f}_t(x_t)H_t^{\nicefrac{1}{2}}u_t$ for 
$x_t = y_t + H_t^{-\nicefrac{1}{2}} u_t$,  $H_t = \nabla^2 \psi(y_t) +\eta_t\left(\sigma_{1:t-1}+\lambda_{1:t-1}\right)\dpI$, and $u_t$ chosen randomly from the unit sphere.

While this forms a natural and basic framework of our algorithm, there is in fact a critical issue when analyzing such a barrier-regularized FTRL algorithm due to the decreasing learning rate, which was never encountered in the literature as far as we know since all related works using this framework adopt a fixed learning rate (see e.g.~\citep{Competing:Dark, AISTATS'11:smooth-BCO, hazan2011better, rakhlin2013online, conf/nips/HazanL14, bubeck2019improved}) or an increasing learning rate~\citep{NIPS'20:unbiased-bandits}.
More specifically, to bound the stability $\inner{y_t-y_{t+1}, g_t}$ of the algorithm, all analysis for barrier-regularized FTRL implicitly or explicitly requires bounding the \emph{Newton decrement}, which in our context is $\|\nabla G_t(y_t)\|_{\nabla^{-2} G_t(y_t)}^2$ for $G_t$ being the objective function in the FTRL update~\pref{eq:FTRL-no-lifting}.
To simplify this term, note that since $\psi$ is a barrier and $y_t$ minimizes $G_{t-1}$, we have $\nabla G_{t-1}(y_t) = 0 = \sum_{s=1}^{t-1} \big(g_s + (\sigma_s+\lambda_s)(y_t - y_s) \big) + \frac{1}{\eta_{t}}\nabla\psi(y_t)$.
Further combining this with $\nabla G_{t}(y_t) = \sum_{s=1}^{t} g_s + \sum_{s=1}^{t-1} (\sigma_s+\lambda_s)(y_t - y_s) + \frac{1}{\eta_{t+1}}\nabla\psi(y_t)$ shows $\nabla G_t(y_t)=g_t+(\frac{1}{\eta_{t+1}}-\frac{1}{\eta_{t}})\nabla\psi(y_t)$.
Now, if $\eta_t$ is fixed for all $t$, then the Newton decrement simply becomes $\|g_t\|_{\nabla^{-2}G_t(y_t)}^2 \preceq \|g_t\|_{\nabla^{-2}G_{t-1}(y_t)}^2$, which by the definition of $g_t$ is directly bounded by $\eta_t d^2$.
However, with decreasing learning rates, the extra term contributes to a term of order $(\frac{1}{\eta_{t+1}}-\frac{1}{\eta_{t}})^2\|\nabla\psi(y_t)\|_{\nabla^{-2}G_t(y_t)}^2$, which could be prohibitively large unfortunately.\footnote{Instead of FTRL, one might wonder if using the highly related Online Mirror Descent framework could solve the issue caused by decreasing learning rates.
We point out that while this indeed addresses the issue for bounding the stability term,
it on the other hand introduces a similar issue for the regularizaton penalty term $\sum_{t=2}^T(\frac{1}{\eta_{t+1}}-\frac{1}{\eta_t})D_{\psi}(x,y_t)$.
}

To resolve this issue, our key observation is: $\|\nabla\psi(y_t)\|_{\nabla^{-2}G_t(y_t)}^2 \preceq \eta_{t+1}\|\nabla\psi(y_t)\|_{\nabla^{-2}\psi(y_t)}^2$,
and $\|\nabla\psi(y_t)\|_{\nabla^{-2}\psi(y_t)}^2$ is always bounded by $\nu$ as long as the self-concordant barrier $\psi$ is also \emph{logarithmically homogeneous} (see \pref{appendix: self-concordant} for definition and \pref{lemma:normal-barrier} for this property). 
A logarithmically homogeneous self-concordant barrier is also called a \emph{normal barrier} for short,
and it is only defined for a cone (recall that our feasible set $\calX$, on the other hand, is always bounded, thus not a cone).
Fortunately, this issue has been addressed in a recent work by~\citep{NIPS'20:unbiased-bandits} on achieving high probability regret bounds for adversarial linear bandits.
Their motivation is different, that is, to bound the variance of the learner's loss, related to $\|y_t\|_{\nabla^{2}\psi(y_t)}^2$ using our notation, and this turns out to be bounded when $\psi$ is a normal barrier --- in fact, $\|y_t\|_{\nabla^{2}\psi(y_t)}^2$ and $\|\nabla\psi(y_t)\|_{\nabla^{-2}\psi(y_t)}^2$ are exactly the same in this case!
Their solution regarding $\calX$ not being a cone is to first lift it to $\R^{d+1}$ and find a normal barrier of the conic hull of this lifted domain (which always exists),
then perform FTRL over the lifted domain with this normal barrier regularizer.
We extend their idea from linear bandits to convex bandits, formally described below (see also \pref{alg:lift-smooth} for the pseudocode).

\paragraph{Lifted domain and normal barrier.}  
To make the dimension of a vector/matrix self-evident, we use bold letters to represent vectors in $\mathbb{R}^{d+1}$ and matrices in $\mathbb{R}^{(d+1)\times (d+1)}$. Define the lifted domain as $\bm{\calX}=\{\dpx=(x,1) \mid x\in \calX\}\subseteq \mathbb{R}^{d+1}$, which simply appends an additional coordinate with constant value $1$ to all points in $\calX$. Define the conic hull of this set as $\calK=\{(x,b) \mid x\in \mathbb{R}^d, b\geq 0, \frac{1}{b}x\in \calX\}$. 
Our algorithm requires using a normal barrier over $\calK$ as a regularizer (which always exists).
While any such normal barrier works, we simply use a canonical one constructed from a $\nu$-self concordant barrier $\psi$ of $\calX$, defined via $\Psi(\dpx)=\Psi(x,b)=400\psi(\nicefrac{x}{b})-2\nu \ln b$ and proven to be a $\Theta(\nu)$-normal barrier over $\calK$ in~\citep[Proposition 5.14]{nesterov1994-IPM}.
This also shows that our algorithm requires no more than that of~\citep{conf/nips/HazanL14}.

\pref{alg:lift-smooth} then performs FTRL in the lifted domain using $\Psi$ as the regularizer to maintain an auxiliary sequence $\dpy_1, \ldots, \dpy_T$; see \pref{line: FTRL update lifted}.
This follows the earlier update rule in ~\pref{eq:FTRL-no-lifting}, except for an additional $\ell_2$ regularization term $\frac{\lambda_0}{2} \norm{\dpx}_2^2$ added for technical reasons.
With $\dpy_t$ at hand, we compute Hessian matrix $\bm{H}_t=\nabla^2{\Psi}(\dplus{y}_t)+\eta_t\left(\sigma_{1:t-1}+\lambda_{0:t-1}\right)\dpI$ similarly as before (\pref{line: hessian calc}).
What is slightly different now is the exploration (\pref{line: final decision}): we sample $\dpu_t$ uniformly at random from the set $\mathbb{S}^{d+1}\cap \big(\dplus{H}_t^{-\frac{1}{2}}\dpe_{d+1}\big)^{\perp}$ where $\dplus{w}^\perp$ denotes the space orthogonal to $\dplus{w}$,
and then obtain a point $\dpx_t= \dpy_t+\dplus{H}_t^{-\frac{1}{2}}\dpu_t$.
It can been shown that $\dpx_t$ is always on the intersection of the lifted domain $\dpX$ and the surface of some ellipsoid centered at $\dpy_t$ --- we refer the reader to~\citep[Figure 1]{NIPS'20:unbiased-bandits} for a pictorial illustration and their description for how to sample $\dpu_t$ efficiently.
Since $\dpx_t \in \bm{\calX}$, it is in the form of $(x_t, 1)$, where $x_t$ will be the final decision of the algorithm.

Upon receiving $f_t(x_t)$ and $\sigma_t$, \pref{alg:lift-smooth} computes the $\ell_2$ regularization coefficient $\lambda_t$ in some way (to be discussed soon), gradient estimator $\dpg_t$ as in earlier discussion (\pref{line: gradient calc}), learning rate $\eta_{t+1}$ as in \pref{line: adaptive learning rate}, and finally $\dpy_{t+1}$ via the aforementioned FTRL.

\begin{algorithm}[!t]
\caption{Adaptive Smooth BCO with Heterogeneous Strong Convexity}
\label{alg:lift-smooth}
\textbf{Input:} smoothness parameter $\beta$ and a $\nu$-self-concordant barrier $\psi$ for the feasible domain $\calX$.

\textbf{Define:} lifted feasible set $\bm{\calX}=\{\bm{x}=(x,1) \mid x\in \calX\}$.

\textbf{Define:} $\Psi$ is a normal barrier of the conic hull of $\dpX$: $\Psi(\x) = \Psi(x,b) = 400(\psi(x/b) - 2\nu \ln b)$. 

\textbf{Define:} $\rho=512\nu(1+32\sqrt{\nu})^2$.

\textbf{Initialize:} $\lambda_0= \max\left\{(\beta+1)\rho\nu^{-1}, d^2(\beta+1)\right\}$ and $\eta_1=\frac{1}{2d}\sqrt{\frac{\beta+1}{\lambda_0}+\frac{\nu}{T\log T}}$.

\textbf{Initialize:} $\dpy_1 =(y_1,1) = \argmin_{\dpx \in \bm{\X}} \Psi(\dpx)$.

\For{$t=1,2,\dots, T$}{
    
    \nl Compute $\bm{H}_t=\nabla^2{\Psi}(\dplus{y}_t)+\eta_t\left(\sigma_{1:t-1}+\lambda_{0:t-1}\right)\dpI$. \label{line: hessian calc}
    
    \nl Draw $\dplus{u}_t$ uniformly at random from $\mathbb{S}^{d+1}\cap \big(\dplus{H}_t^{-\frac{1}{2}}\dpe_{d+1}\big)^{\perp}$. 
    \LineComment{$\dplus{w}^\perp$: space orthogonal to $\dplus{w}$}  \label{line: exploration}
    
    \nl Compute $\dpx_t=(x_t,1) = \dpy_t+\dplus{H}_t^{-\frac{1}{2}}\dpu_t$, 
    play the point $x_t$, and observe $f_t(x_t)$ and $\sigma_t$. \label{line: final decision}

    \nl Compute regularization coefficient $\lambda_t \in (0,1)$. \LineComment{See \pref{eqn:oracle-tuning-strongly-main} and related discussions}
    
    \nl Compute gradient estimator $\dpg_t=d\left(f_t(x_t)+\frac{\lambda_t}{2}\|x_t\|_2^2\right)\dplus{H}_t^{\frac{1}{2}}\dpu_t$. \label{line: gradient calc}
    
    \nl Compute learning rate $\eta_{t+1} = \frac{1}{2d}\sqrt{\frac{\beta+1}{\sigma_{1:t}+\lambda_{0:t}}+\frac{\nu}{T\log T}}$. \label{line: adaptive learning rate}
    
    \nl Update $\dpy_{t+1}=\argmin_{\dpx \in \dplus{\calX}}\left\{\sum_{s=1}^t \left(\dpg_{s}^\top \dpx + \frac{\sigma_{s}+\lambda_s}{2}\|\dpx-\dpy_{s}\|_2^2\right) + \frac{\lambda_0}{2} \norm{\dpx}_2^2 + \frac{1}{\eta_{t+1}}\Psi(\dpx)\right\}$. \label{line: FTRL update lifted}
}
\end{algorithm}

\paragraph{Guarantees and $\ell_2$ regularization coefficient tuning.}
We now present some guarantees of our algorithm that hold regardless of the tuning of $\{\lambda_t\}_{t=1}^T$.
First, we show that except for the last coordinate, $\dpg_t$ is an unbiased estimator of a smoothed version of $\wt{\dpf}_t(\dpy_t)$, where $\wt{\dpf}_t(\dpy) = f_t(\dpy_{[1:d]}) + \frac{\lambda_t}{2}\|\dpy_{[1:d]}\|_2^2$.
This is a non-trivial generalization of~\citep[Lemma~B.9]{NIPS'20:unbiased-bandits} from linear functions to convex functions. See \pref{appendix:proof-lemma-unbiased} for the proof.
\begin{lemma}
\label{lem: unbiased}
For each $t\in [T]$, we have $\mathbb{E}_t\left[\dpg_{t,i}\right]=\nabla \wh{\dpf}_t(\dpy_t)_{[i]}$ for all $i\in [d]$, where $\wh{\dpf}_t$ is the smoothed version of $\wt{\dpf}_t$ defined as
    $\wh{\dpf}_t(\dpx) \triangleq \mathbb{E}_{\dpb}\big[\wt{\dpf}_t(\dpx+\dpH_t^{-\nicefrac{1}{2}}\dpb)\big]$, 
where $\dpb$ is uniformly sampled from $\mathbb{B}^{d+1}\cap (\dpH_t^{-\frac{1}{2}}\dpe_{d+1})^\perp$.
\end{lemma}

Thanks to the unbiasedness of the gradient estimators and the crucial properties of normal barrier, we prove the following regret guarantee of \pref{alg:lift-smooth} in~\pref{appendix:proof-smooth-arbitrary}.
\begin{lemma}
\label{lemma:main-result-smooth-arbitrary}
With any regularization coefficients $\{\lambda_t\}_{t=1}^T \in (0,1)$,
\pref{alg:lift-smooth} guarantees:
\begin{equation}
\label{eqn: strongly-theorem-arbitrary}
    \Reg = \Ot\left(d \sqrt{\nu T}+ \lambda_{1:T-1} + \sum_{t=1}^{T-1}\frac{d\sqrt{\beta+1}}{\sqrt{\sigma_{1:t} + \lambda_{0:t}}}\right),
\end{equation}
if loss functions $\{f_t\}_{t=1}^T$ are all $\beta$-smooth and $T\geq \rho$ (a constant defined in \pref{alg:lift-smooth}).
\end{lemma}

We are now in the position to specify the tuning of the regularization coefficients.
Based on the bound in \pref{eqn: strongly-theorem-arbitrary}, we propose to balance the last two terms by picking $\lambda_t \in (0,1)$ such that:
\begin{align}\label{eqn:oracle-tuning-strongly-main}
    \lambda_t  = \frac{d\sqrt{\beta+1}}{\sqrt{\sigma_{1:t}+\lambda_{0:t}}},
\end{align}
which must exist since when $\lambda_t = 0$, the left-hand side is smaller than the right-hand side, while when $\lambda_t = 1$, the left-hand side is larger than the right-hand side by the definition of $\lambda_0$.
Note that unlike the AOGD tuning in \pref{eq:AOGD_tuning}, our tuning leads to a cubic equation of $\lambda_t$, which does not admit a closed-form.
However, the earlier argument on its existence clearly also implies that it can be computed via a simple and efficient binary search (using information available at the end of round $t$).
Our next lemma is in the same vein as~\citep[Lemma~3.1]{NIPS'07:AOGD},
which shows that our adaptive tuning is almost as good as the optimal tuning (that knows all $\sigma_t$'s ahead of time).
\begin{lemma}
\label{lemma:H-function-main}
Define
$\calB(\{\lambda_s\}_{s=1}^t) \triangleq \lambda_{1:t}+\sum_{\tau=1}^t\frac{d\sqrt{\beta+1}}{\sqrt{\sigma_{1:\tau}+\lambda_{0:\tau}}}$,
with $\lambda_0$ defined in \pref{alg:lift-smooth}.
Then the sequence $\{\lambda_t\}_{t=1}^T$ attained by solving~\pref{eqn:oracle-tuning-strongly-main} satisfies for all $t \in [T]$:
\begin{equation}
  \label{eq:near-optimal-H-function}
  \calB(\{\lambda_s\}_{s=1}^t)\leq 2\min_{\{\lambda_s^*\}_{s=1}^t\geq 0}\calB(\{\lambda_s^*\}_{s=1}^t).
\end{equation} 
\end{lemma}

The proof of \pref{lemma:H-function-main} is deferred in \pref{appendix:H-function-main}. Combining \pref{lemma:main-result-smooth-arbitrary} and \pref{lemma:H-function-main}, we obtain the final regret guarantee in~\pref{thm:main-result-smooth}, whose proof can be found in \pref{appendix:proof-main-result-smooth}.
\begin{theorem}
\label{thm:main-result-smooth}
\pref{alg:lift-smooth} with adaptive tuning~\pref{eqn:oracle-tuning-strongly-main} ensures for any sequence $\lambda_1^*,\ldots,\lambda_T^* \geq 0$:
\begin{equation}\label{eqn: strongly-theorem}
    \Reg = \Ot\left(d\sqrt{\nu T} + \lambda_{1:T-1}^* +  \sum_{t=1}^{T-1}\frac{d\sqrt{\beta+1}}{\sqrt{\sigma_{1:t} + \lambda_{0:t}^*}}\right),
\end{equation}
when loss functions $\{f_t\}_{t=1}^T$ are all $\beta$-smooth and $T\geq \rho$ (a constant defined in \pref{alg:lift-smooth}).
\end{theorem}

We leave the discussion on the many implications of this general regret bound to the next subsection, and make a final remark on the per-round computational complexity of \pref{alg:lift-smooth}.
Note that although the objective function in the FTRL update (\pref{line: FTRL update lifted}) contains $\order(T)$ terms, it is clear that by storing and updating some statistics (such as $\sum_{s=1}^t \dpg_s$ and $\sum_{s=1}^t (\sigma_s+\lambda_s)\dpy_s$), one can evaluate its function value and gradient in time independent of $T$.
Approximating solving the FTRL update (up to precision $\nicefrac{1}{\poly(T)}$) via for example the interior point method thus only requires $\order(\poly(d\log T))$ time.
This is more efficient than the method of~\citet{JACM'21:kernel-BCO}, which requires $\order(\poly(d\log T)T)$ time per round even when the domain is a polytope.

\subsection{Implications of \pref{thm:main-result-smooth}}
\label{sec: implication}
In the following, we investigate several special cases and present direct corollaries of \pref{thm:main-result-smooth} to demonstrate that our algorithm not only matches/improves existing results for homogeneous settings, but also leads to interesting intermediate results in some heterogeneous settings. 
Note that since our regret bound in \pref{thm:main-result-smooth} holds for \emph{any} choice of the sequence $\{\lambda_t^*\}_{t=1}^T\geq 0$, in each case below we will simply provide a specific sequence of $\{\lambda_t^*\}_{t=1}^T$ that leads to a favorable guarantee.
For simplicity, we also directly replace $\nu$ with $\order(d)$ (in both our bounds and previous results) since it is well known that any convex set admits an $\order(d)$-self-concordant barrier~\citep{nesterov1994-IPM}.

First, consider the case when no functions have strong convexity, that is, $\sigma_t = 0$ for all $t$.
This degenerates to the same homogeneous setting as~\citep{AISTATS'11:smooth-BCO}, where their algorithm achieves $\Ot(d^{\nicefrac{3}{2}}\sqrt{T}+\beta^{\nicefrac{1}{3}}d T^{\nicefrac{2}{3}})$ regret. As a side product of our adaptive $\ell_2$-regularization, our algorithm manages to achieve even better dependency on the dimension $d$. Indeed, by choosing $\lambda_1^*=(1+\beta)^{\nicefrac{1}{3}}d^{\nicefrac{2}{3}}T^{\nicefrac{2}{3}}$ and $\lambda_t^*=0$ for all $t\geq 2$ in~\pref{thm:main-result-smooth}, we obtain the following corollary.
\begin{corollary}[Smooth BCO without strong convexity]\label{cor: only-smooth}
When $f_t$ is $\beta$-smooth and $0$-strongly convex for all $t \in[T]$, \pref{alg:lift-smooth} achieves $\otil(d^{\nicefrac{3}{2}}\sqrt{T}+(1+\beta)^{\nicefrac{1}{3}}d^{\nicefrac{2}{3}}T^{\nicefrac{2}{3}})$ regret.
\end{corollary}

Second, we consider the case when all functions are $\sigma$-strongly convex for some constant $\sigma > 0$, which degenerates to the same homogeneous setting as~\citep{conf/nips/HazanL14}.
By picking $\lambda^*_t=0$ for all $t\geq 0$ in~\pref{thm:main-result-smooth}, our algorithm achieves the same result as theirs.
\begin{corollary}[Smooth BCO with $\sigma$-strong convexity]
\label{cor: strongly-convex-smooth}
When $f_t$ is $\beta$-smooth and $\sigma$-strongly convex, i.e., $\sigma_{t}=\sigma > 0$ for all $t \in[T]$, \pref{alg:lift-smooth} achieves $\Ot(d^{\nicefrac{3}{2}}\sqrt{T}+d\sqrt{\nicefrac{T(1+\beta)}{\sigma}})$ regret.
\end{corollary}

Third, we investigate an intermediate setting with a mixture of $\sigma$-strongly convex and $0$-strongly convex functions. 
Specifically, suppose that there are $M$ functions with no strong convexity, and the rest are $\sigma$-strongly convex.
According to \pref{thm:main-result-smooth}, the worst case scenario for our algorithm is when these $M$ functions appear in the first $M$ rounds, while the best scenarios is when they are in the last $M$ rounds.
Considering these two extremes and picking $\{\lambda_t^*\}_{t=1}^T$ correspondingly, we obtain the following corollary (see \pref{appendix: smooth implication proof} for the proof).

\begin{corollary}[Smooth BCO with a mixture of convex and $\sigma$-strongly convex functions]\label{cor: worst-best-case}
Suppose that $\{f_t\}_{t=1}^T$ are $\beta$-smooth
and $T-M$ of them are $\sigma$-strongly convex.
Then \pref{alg:lift-smooth} guarantees
$
    \Reg= \otil \Big(d^{\frac{2}{3}}(1+\beta)^{\frac{1}{3}} M^{\frac{2}{3}}+d^{\frac{3}{2}}\sqrt{T}+d\sqrt{\frac{(1+\beta)(T-M)}{\sigma}}
    \Big).
$
If these $T-M$ functions appear in the first $T-M$ rounds, then the bound is further improved to
$
    \Reg= \otil \Big(d^{\frac{3}{2}}\sqrt{T}+dT\sqrt{\frac{1+\beta}{\sigma(T-M)}}
    \Big).
$
\end{corollary}
To better interpret these bounds, we consider how large $M$ can be (that is, how many functions without strong convexity we can tolerate) to still ensure $\otil(\sqrt{T})$ regret
---
in the general case (the first bound of the corollary), we see that we can tolerate $M=\order(T^{\nicefrac{3}{4}})$,
while in the best case (the second bound), we can even tolerate $M$ being any constant fraction of $T$!
On the other hand, a naive method of discarding all functions without strong convexity can only tolerate $M = \otil(\sqrt{T})$.

Finally, following~\citep{NIPS'07:AOGD} we consider a situation with decaying strong convexity: $\sigma_{t}=t^{-\alpha}$ for some $\alpha \in [0,1]$. 
We prove the following corollary; see \pref{appendix: smooth implication proof} for the proof.

\begin{corollary}[Smooth BCO with decaying strong convexity]\label{cor: intermediate-case}
When $f_t$ is $\beta$-smooth and $\sigma_t$-strongly convex with $\sigma_t=t^{-\alpha}$ for some $\alpha \in [0,1]$, \pref{alg:lift-smooth} guarantees
\begin{align*}
    \Reg= 
    \begin{cases}
        \otil\left(d^{\frac{3}{2}}\sqrt{T}+d\sqrt{1+\beta}T^{\frac{1+\alpha}{2}}\right) & \alpha\in [0,\frac{1}{3}-\frac{2}{3}\log_Td-\frac{1}{3}\log_T(1+\beta)], \\
        \otil\left(d^{\frac{3}{2}}\sqrt{T}+(1+\beta)^{\frac{1}{3}}d^{\frac{2}{3}}T^{\frac{2}{3}}\right) & \alpha\in [\frac{1}{3}-\frac{2}{3}\log_Td-\frac{1}{3}\log_T(1+\beta), 1].
    \end{cases}
\end{align*}
\end{corollary}

\section{Lipschitz Bandit Convex Optimization with Heterogeneous Strong Convexity}
\label{sec:lipschitz-BCO}

\setcounter{AlgoLine}{0}
\begin{algorithm}[t]
\caption{Adaptive Lipschitz BCO with Heterogeneous Strong Convexity}
\label{alg:lift-Lipschitz}
\textbf{Input:} Lipschitz parameter $L$ and a $\nu$-self-concordant barrier $\psi$ for the feasible domain $\calX$.

\textbf{Define:} lifted feasible set $\bm{\calX}=\{\bm{x}=(x,1) \mid x\in \calX\}$.

\textbf{Define:} $\Psi$ is a normal barrier of the conic hull of $\dpX$: $\Psi(\x) = \Psi(x,b) = 400(\psi(x/b) - 2\nu \ln b)$. 

\textbf{Define:} $\rho'=\frac{2^{16}(16\sqrt{\nu} d^{\nicefrac{1}{3}}(4L+1)^{\nicefrac{1}{3}}+(L+1)^{\nicefrac{2}{3}})^3}{d}$.

\textbf{Initialize:} $\lambda_0=\max\{ \rho', d^2(L+1)^2\}$ and $\eta_1=(L+1)^{\frac{2}{3}}d^{-\frac{4}{3}}(\frac{1}{\lambda_0}+\frac{1}{T})^{\frac{1}{3}}$.

\textbf{Initialize:} $\dpy_1 =(y_1,1) = \argmin_{\dpx \in \bm{\X}} \Psi(\dpx)$.

\For{$t=1,2,\dots, T$}{
    
    \nl Define $\bm{H}_t=\nabla^2{\Psi}(\dplus{y}_t)+\eta_t\left(\sigma_{1:t-1}+\lambda_{0:t-1}\right)\dpI$. \label{line: hessian calc lip}
    
    \nl Draw $\dplus{u}_t$ uniformly at random from $\mathbb{S}^{d+1}\cap \big(\dplus{H}_t^{-\frac{1}{2}}\dpe_{d+1}\big)^{\perp}$. 
    \LineComment{$\dplus{w}^\perp$: space orthogonal to $\dplus{w}$} \label{line: exploration lip}
    
    \nl Compute $\dpx_t=(x_t,1) = \dpy_t+\dplus{H}_t^{-\frac{1}{2}}\dpu_t$, 
    play the point $x_t$, and observe $f_t(x_t)$ and $\sigma_t$. \label{line: final decision lip}

    \nl Compute regularization coefficient $\lambda_t \in (0,1)$ as the solution of the following equation \label{line: adaptive lambda lip}
    \begin{align}\label{eqn:oracle-tuning-Lipschitz-main}
        \lambda_t = \frac{d^{\frac{2}{3}}(L+1)^{\frac{2}{3}}}{(\sigma_{1:t}+\lambda_{0:t})^{\frac{1}{3}}}.
    \end{align}
    
    \nl Compute gradient estimator $\dpg_t=d\left(f_t(x_t)+\frac{\lambda_t}{2}\|x_t\|_2^2\right)\dplus{H}_t^{\frac{1}{2}}\dplus{u}_t$. \label{line: gradient calc lip}
    
    \nl Compute learning rate $\eta_{t+1} = d^{-\frac{4}{3}}(L+1)^{\frac{2}{3}}\cdot\left({\frac{1}{\sigma_{1:t}+\lambda_{0:t}}+\frac{1}{T}}\right)^{\frac{1}{3}}$. \label{line: adaptive learning rate lip}
    
    \nl Update $\dpy_{t+1}=\argmin_{\dpx\in \dplus{\calX}}\left\{\sum_{s=1}^t \left(\dpg_{s}^\top \dpx + \frac{\sigma_{s}+\lambda_s}{2}\|\dpx-\dpy_{s}\|_2^2\right) + \frac{\lambda_0}{2}\|\dpx\|_2^2+\frac{1}{\eta_{t+1}}\Psi(\dpx)\right\}$. \label{line: FTRL update lifted lip}
}
\end{algorithm}

In this section, we consider a similar setting where instead of assuming smoothness, we assume that functions $\{f_t\}_{t=1}^T$ are known to be $L$-Lipschitz (\pref{assumption:lipschitz}).
The strong convexity parameter $\sigma_t$ of function $f_t$ is still only revealed at the end of round $t$.
We extend our algorithm to this case and present it in \pref{alg:lift-Lipschitz}, which differs from \pref{alg:lift-smooth} only in the tuning of the learning rate $\eta_t$ (see \pref{line: adaptive learning rate lip}) and the regularization coefficient $\lambda_t$ (see \pref{line: adaptive lambda lip}).
These tunings are different because of the different structures in the setting, but their design follows the same idea as before.
Similar to~\pref{thm:main-result-smooth}, we prove the following theorem (see \pref{appendix: lip-bco-proof} for the proof).
\begin{theorem}
\label{thm:main-result-lipschitz}
\pref{alg:lift-Lipschitz} ensures for any sequence $\lambda_1^*,\ldots,\lambda_T^* \geq 0$:
\begin{align}
\label{eqn: lip regret bound}
    \Reg =  \otil\left(d^{\frac{4}{3}}\nu T^{\frac{1}{3}}+\lambda_{0:T}^*+\sum_{t=1}^T\frac{d^{\frac{2}{3}}(L+1)^{\frac{2}{3}}}{(\sigma_{1:t-1}+\lambda_{0:t-1}^*)^{\frac{1}{3}}}\right),
\end{align}
when all the functions are $L$-Lipschitz and $T\geq \rho'$ (a constant defined in~\pref{alg:lift-Lipschitz}).
\end{theorem}

Similar to~\pref{sec: implication}, we now discuss the implications of this theorem in several special cases, demonstrating that our algorithm not only improves existing results in the homogeneous settings as a side product of the better regularization technique, but also achieves favorable guarantees in some heterogeneous settings.
Again, we plug in $\nu = \order(d)$ for simplicity.

First, we consider the case when no functions have strong convexity, which degenerates to the same homogeneous setting studied in~\citep{conf/nips/Kleinberg04,SODA'05:Flaxman-BCO,lecture18}.
Among these results, the best regret bound is $\otil(\sqrt{L}d^{\nicefrac{3}{4}}T^{\nicefrac{3}{4}})$~\citep{lecture18}.
By picking $\lambda_1^*=\sqrt{d(L+1)}T^{\nicefrac{3}{4}}$ and $\lambda_t^*=0$ for all $t\geq 2$ in~\pref{thm:main-result-lipschitz}, we achieve the following guarantee with improved dimension dependency.
\begin{corollary}[Lipschitz BCO without strong convexity]\label{cor: only-Lipschitz}
When $f_t$ is $L$-Lipschitz and $0$-strongly convex for all $t \in[T]$, \pref{alg:lift-Lipschitz} achieves $\otil(\sqrt{d(L+1)}T^{\nicefrac{3}{4}})$ regret.
\end{corollary}

Second, we consider the case when all loss functions are $\sigma$-strongly convex. This degenerates to the homogeneous setting studied in~\citep{COLT'10:agarwal_optimal}, where they achieve $\otil(d^{\nicefrac{4}{3}}L^{\nicefrac{2}{3}}\sigma^{-\nicefrac{1}{3}}T^{\nicefrac{2}{3}})$ regret.\footnote{The bound stated in their paper has $d^{\nicefrac{2}{3}}$ dependency on the dimension, but that is under a different assumption on $\calX$. 
Translating their setting to ours via a reshaping trick~\citep[Section 3.2]{SODA'05:Flaxman-BCO} leads to the $d^{\nicefrac{4}{3}}$ dependency.}  
Once again, by picking $\lambda_t^*=0$ for all $t\geq 1$ in~\pref{thm:main-result-lipschitz}, we obtain the following result with improved dimension dependency.
\begin{corollary}[Lipschitz BCO with $\sigma$-strong convexity]\label{cor: only-Lipschitz-strongly-cvx}
When $f_t$ is $L$-Lipschitz and $\sigma$-strongly convex, i.e., $\sigma_{t}=\sigma$ for all $t \in[T]$, \pref{alg:lift-Lipschitz} achieves $\otil(d^{\nicefrac{2}{3}}(L+1)^{\nicefrac{2}{3}}\sigma^{-\nicefrac{1}{3}}T^{\nicefrac{2}{3}})$ regret.
\end{corollary}

Third, we consider the case with a mixture of $0$-strongly convex and $\sigma$-strongly convex functions.
\begin{corollary}[Lipschitz BCO with a mixture of convex and $\sigma$-strongly convex functions]\label{cor: only-lipschitz-worst-best-case}
Suppose that $\{f_t\}_{t=1}^T$ are $L$-Lipschitz and $T-M$ of them are $\sigma$-strongly convex.
Then \pref{alg:lift-Lipschitz} guarantees
$
    \Reg= \otil\Big(\sqrt{d(L+1)}M^{\nicefrac{3}{4}}+d^{\nicefrac{2}{3}}(L+1)^{\nicefrac{2}{3}}\sigma^{-\nicefrac{1}{3}}(T-M)^{\nicefrac{2}{3}}
    \Big).
$
If these $T-M$ functions appear in the first $T-M$ rounds, then the bound is further improved to
$
    \Reg= \otil \Big(\frac{d^{\nicefrac{2}{3}}(L+1)^{\nicefrac{2}{3}}T}{\sigma^{\nicefrac{1}{3}}(T-M)^{\nicefrac{1}{3}}}
    \Big).
$
\end{corollary}

The proof can be found in \pref{appendix: lip implication proof}.
Similar to the discussion in~\pref{sec: implication}, we consider how large $M$ can be to still ensure $\otil(T^{\nicefrac{2}{3}})$ regret ---
according to the first bound, we can always tolerate $M=\order(T^{\nicefrac{8}{9}})$,
while in the best case (the second bound), we can tolerate $M$ being any constant fraction of $T$.
These are again much stronger compared to the naive method of discarding all functions without strong convexity, which can only tolerate $M = \otil(T^{\nicefrac{2}{3}})$.

Finally, we consider the example with $\sigma_t=t^{-\alpha}$ again. See  \pref{appendix: lip implication proof} for the proof.
\begin{corollary}[Lipschitz BCO with decaying strong convexity]
\label{cor: only-lipschitz-intermediate-case}
When $f_t$ is $L$-Lipschitz and $\sigma_t$-strongly convex with $\sigma_t=t^{-\alpha}$ for some $\alpha \in [0,1]$, \pref{alg:lift-Lipschitz} guarantees
\begin{align*}
    \Reg= 
    \begin{cases}
        \otil(d^{\frac{2}{3}}(L+1)^{\frac{2}{3}}T^{\frac{2+\alpha}{3}}) & \alpha\in [0, \frac{1}{4}-\frac{1}{2}\log_T(L+1)-\frac{1}{2}\log_Td], \\
        \otil(\sqrt{d(L+1)}T^{\frac{3}{4}}) & \alpha\in [\frac{1}{4}-\frac{1}{2}\log_T(L+1)-\frac{1}{2}\log_Td, 1].
    \end{cases}
\end{align*}
\end{corollary}

\section{Conclusion}
Our work initiates the study of bandit convex optimization with heterogeneous curvature and proposes strong algorithms and guarantees that automatically adapt to individual curvature of each loss function.
As the first step in this direction, we have assumed homogeneous smoothness or Lipschitzness and only considered heterogeneous strong convexity.
Extending the heterogeneity to the other curvature parameters is an immediate next step.
Moreover, it is worth investigating an even more challenging setting where the individual curvature information is not revealed to the learner at the end of each round, or at least has to be learned via other weaker and indirect feedback (such as some rough and potentially incorrect estimation of the curvature).

\acks{Peng Zhao is supported by NSFC (61921006). HL and MZ are supported by NSF Award IIS-1943607.}

\bibliography{bibfile}

\appendix
\section{Omitted Details for \pref{sec:smooth-BCO}}
\label{appendix:details-smooth}

\subsection{Proof of \pref{lem: unbiased}}
\label{appendix:proof-lemma-unbiased}
In this section, we prove one of our key lemmas (\pref{lem: unbiased}), which shows that the lifted gradient estimator constructed in~\pref{line: gradient calc} of \pref{alg:lift-smooth} is an unbiased estimator in the first $d$ dimensional coordinates. 

\begin{proof}
Fix any $t\in [T]$ and let $\dpw=\dpH_t^{-\frac{1}{2}}\dpe_{d+1}/\big\|\dpH_t^{-\frac{1}{2}}\dpe_{d+1}\big\|_2$. As $\dpu \sim \mathbb{S}^{d+1}\cap \big(\dplus{H}_t^{-\frac{1}{2}}\dpe_{d+1}\big)^{\perp}$ and $\dpb \sim \mathbb{B}^{d+1}\cap \big(\dplus{H}_t^{-\frac{1}{2}}\dpe_{d+1}\big)^{\perp}$, there exists a transformation matrix $\dpM\in \mathbb{R}^{d\times (d+1)}$ that satisfies $\dpM\dpw=0$, $\dpM^\top\dpM = \dpI-\dpw\dpw^\top$ and $\dpM\dpM^\top=I$, such that $\dpu=\dpM^\top v$, $\dpb=\dpM^\top b$ where $v$ is uniformly drawn from $\mathbb{S}^d$ and $b$ is uniformly drawn from $\mathbb{B}^d$. In fact, the $d$ row vectors of $\dpM$ together with $\dpw$ forms a set of unit orthogonal base in the $(d+1)$-dimensional space.

Recall the two following functions whose feasible domain is in $(d+1)$-dimensional space.
\begin{align}
    &\dpf_t(\dpx) \triangleq f_t(\dpx_{[1:d]}),\label{eqn:lifted_original}\\
    &\wt{\dpf}_t(\dpx) \triangleq \dpf_t(\dpx)+\frac{\lambda_t}{2}\|\dpx_{[1:d]}\|_2^2 = f_t(\dpx_{[1:d]})+\frac{\lambda_t}{2}\|\dpx_{[1:d]}\|_2^2 \label{eqn:lifted_regularize}
\end{align}
Then we define the following functions in the $d$-dimensional space. Let $J_t:\mathbb{R}^d\rightarrow \mathbb{R}$ such that $J_t(x)=\wt{\dpf}_t(\dpH_t^{-\frac{1}{2}}\dpM^\top x)$ and $\wh{J}_t(x)=\mathbb{E}_{b\sim \mathbb{B}^d}[J_t(x+b)]$. In addition, we denote $\wh{\dpy}_t = (y_t,0) \in \R^{d+1}$ that appends an additional constant value $0$ to $y_t$ in the $(d+1)$-th coordinate. Then by the definition of $\dpg_t$, we have
\begin{align*}
    \mathbb{E}_t[\dpg_t]&=d\mathbb{E}_{\dpu\sim \mathbb{S}^{d+1}\cap\big(\dpH_t^{-\frac{1}{2}}\dpe_{d+1}\big)^\perp}\left[\wt{\dpf}_t(\dpy_t+\dpH_t^{-\frac{1}{2}}\dpu)\dpH_t^{\frac{1}{2}}\dpu\right] \\
    &=d\mathbb{E}_{v\sim \mathbb{S}^d}\left[J_t(\dpM\dpH_t^{\frac{1}{2}}\wh{\dpy}_t+v)\dpH_t^{\frac{1}{2}}\dpM^\top v\right] \\
    &=d\dpH_t^{\frac{1}{2}}\dpM^\top\mathbb{E}_{v\sim \mathbb{S}^d}\left[J_t(\dpM\dpH_t^{\frac{1}{2}}\wh{\dpy}_t+v)v\right] \\
    &=\dpH_t^{\frac{1}{2}}\dpM^\top\nabla \wh{J}_t(\dpM\dpH_t^{\frac{1}{2}}\wh{\dpy}_t),
\end{align*}
where the final equality is due to Lemma 5 of~\citet{SODA'05:Flaxman-BCO}. The second equality is because of the following reasoning. Note that by the definition of $J_t$ and properties of $\dpM$, we have
\begin{align*}
    J_t(\dpM\dpH_t^{\frac{1}{2}}\wh{\dpy}_t+v) &= \wt{\dpf}_t(\dpH_t^{-\frac{1}{2}}\dpM^\top (\dpM\dpH_t^{\frac{1}{2}}\wh{\dpy}_t+v)) \\
    &=\wt{\dpf}_t(\dpH_t^{-\frac{1}{2}}(\dpI-\dpw\dpw^\top)\dpH_t^{\frac{1}{2}}\wh{\dpy}_t+\dpH_t^{-\frac{1}{2}}\dpM^\top v) \\
    &=\wt{\dpf}_t((\dpI-\dpH_t^{-\frac{1}{2}}\dpw\dpw^\top\dpH_t^{\frac{1}{2}})\wh{\dpy}_t+\dpH_t^{-\frac{1}{2}}\dpM^\top v).
\end{align*}
In addition, by the definition of $\dpw$, we have 
\begin{align*}
    \dpI-\dpH_t^{-\frac{1}{2}}\dpw\dpw^\top\dpH_t^{\frac{1}{2}} = \dpI-\frac{ \dpH_t^{-1}\dpe_{d+1}\dpe_{d+1}^\top}{\big\|\dpH_t^{-\frac{1}{2}}\dpe_{d+1}\big\|_2^2}
\end{align*}

Note that the second term has all entries $0$ except for the last column, i.e., the $(d+1)$-th one. Therefore we have $\dpH_t^{-\frac{1}{2}}\dpw\dpw^\top\dpH_t^{\frac{1}{2}}\wh{\dpy}_t=\bm{0}$, which leads to
\begin{align*}
     J_t(\dpM\dpH_t^{\frac{1}{2}}\wh{\dpy}_t+v) = \wt{\dpf}_t(\wh{\dpy}_t+\dpH_t^{-\frac{1}{2}}\dpM^\top v) = \wt{\dpf}_t(\dpy_t+\dpH_t^{-\frac{1}{2}}\dpM^\top v),
\end{align*}
where the last equality is because $\wt{\dpf}_t(\dpx)$ has no dependence on the last coordinate of $\dpx$. 

Furthermore, according to the definition of $\wh{J}_t(x)$, we have
\begin{align*}
    \nabla \wh{J}_t(x)
    &=\big(\dpH_t^{-\frac{1}{2}}\dpM^\top\big)^\top \mathbb{E}_{b\sim \mathbb{B}^d}\left[\nabla \wt{\dpf}_t(\dpH_t^{-\frac{1}{2}}\dpM^\top (x+b))\right]\\
    &=\big(\dpH_t^{-\frac{1}{2}}\dpM^\top\big)^\top \mathbb{E}_{b\sim \mathbb{B}^d}\left[\nabla \wt{\dpf}_t(\dpH_t^{-\frac{1}{2}}\dpM^\top x+\dpH_t^{-\frac{1}{2}}\dpM^\top b)\right]\\
    &=\big(\dpH_t^{-\frac{1}{2}}\dpM^\top\big)^\top \E_{\dpb \sim \mathbb{B}^{d+1}\cap \big(\dplus{H}_t^{-\frac{1}{2}}\dpe_{d+1}\big)^{\perp}}\left[\nabla \wt{\dpf}_t(\dpH_t^{-\frac{1}{2}}\dpM^\top x+\dpH_t^{-\frac{1}{2}}\dpb)\right]\\
    &= \dpM\dpH_t^{-\frac{1}{2}}\nabla \wh{\dpf}_t(\dpH_t^{-\frac{1}{2}}\dpM^\top x),
\end{align*}
where the fourth equality is by the definition of $\wh{\dpf}_t$. Therefore, we get
\begin{align}
    \E_t[\dpg_t] &= \dpH_t^{\frac{1}{2}}\dpM^\top\nabla \wh{J}_t(\dpM\dpH_t^{\frac{1}{2}}\wh{\dpy}_t) \nonumber \\
    & = \dpH_t^{\frac{1}{2}}\dpM^\top \dpM\dpH_t^{-\frac{1}{2}}\nabla \wh{\dpf}_t(\dpH_t^{-\frac{1}{2}}\dpM^\top\dpM\dpH_t^{\frac{1}{2}}\wh{\dpy}_t) \nonumber\\
    & = (\dpI-\dpH_t^{\frac{1}{2}}\dpw\dpw^\top\dpH_t^{-\frac{1}{2}})\nabla \wh{\dpf}_t((\dpI-\dpH_t^{-\frac{1}{2}}\dpw\dpw^\top\dpH_t^{\frac{1}{2}})\wh{\dpy}_t). \label{eq:unbiased-1}
\end{align}
Since $\dpw = \dpH_t^{-\frac{1}{2}}\dpe_{d+1} \cdot \big\|\dpH_t^{-\frac{1}{2}}\dpe_{d+1}\big\|_2^{-1}$, we have
\begin{align*}
\dpI-\dpH_t^{\frac{1}{2}}\dpw\dpw^\top\dpH_t^{-\frac{1}{2}} = \dpI-\frac{\dpe_{d+1}\dpe_{d+1}^\top \dpH_t^{-1}}{\big\|\dpH_t^{-\frac{1}{2}}\dpe_{d+1}\big\|_2^2}, \quad \dpI-\dpH_t^{-\frac{1}{2}}\dpw\dpw^\top\dpH_t^{\frac{1}{2}} = \dpI-\frac{ \dpH_t^{-1}\dpe_{d+1}\dpe_{d+1}^\top}{\big\|\dpH_t^{-\frac{1}{2}}\dpe_{d+1}\big\|_2^2}.
\end{align*}
This shows that for any $\dpx\in \R^{d+1}$, $\big((\dpI-\dpH_t^{\frac{1}{2}}\dpw\dpw^\top\dpH_t^{-\frac{1}{2}})\dpx\big)_{[1:d]}=\dpx_{[1:d]}$ and also we have $(\dpI-\dpH_t^{-\frac{1}{2}}\dpw\dpw^\top\dpH_t^{\frac{1}{2}})\wh{\dpy}_t=\wh{\dpy}_t$ because the $(d+1)$-th coordinate of $\wh{\dpy}_t$ is $0$. Combining the above with~\pref{eq:unbiased-1} yields the following result:
\begin{align*}
    \mathbb{E}_t[\dpg_t] &= \left(\dpI-\frac{\dpe_{d+1}\dpe_{d+1}^\top\dpH_t^{-1}}{\big\|\dpH_t^{-\frac{1}{2}}\dpe_{d+1}\big\|_2^2}\right)\nabla \wh{\dpf}_t(\wh{\dpy}_t)\\ 
    &= \left(\dpI-\frac{\dpe_{d+1}\dpe_{d+1}^\top\dpH_t^{-1}}{\big\|\dpH_t^{-\frac{1}{2}}\dpe_{d+1}\big\|_2^2}\right) \Big[\nabla \wh{\dpf}_t(\dpy_t)_{[1:d]}; 0] =\left[\nabla \wh{\dpf}_t(\dpy_t)_{[1:d]}; *\right],
\end{align*}
where $* \in \R$ denotes the last coordinate of the expectation of the gradient estimator that can be calculated according to the context. Note that the last step is true by noting that $\wh{\dpf}_t$ is defined as a smoothed function of $\tilde{\dpf}_t$ that is irrelevant to the $(d+1)$-th coordinate. 

Hence, we show that the first $d$ dimensions of the estimator constructed in~\pref{line: gradient calc} are unbiased and  finish the proof.
\end{proof}

\subsection{Stability Lemma}\label{appendix:proof-stability}
In this section, we prove the following lemma which shows the stability of the dynamics of our algorithm. We point out that this stability lemma is the main technical reason that we introduce the lifting idea.
\begin{lemma}\label{lem: stablity 1/2}
Consider the following FTRL update:
\begin{align*}
    \dpy_{t+1}=\argmin_{\dpx \in \dplus{\calX}}\left\{\sum_{s=1}^t \left(\dpg_{s}^\top \dpx + \frac{\sigma_{s}+\lambda_s}{2}\|\dpx-\dpy_{s}\|_2^2\right) + \frac{\lambda_0}{2} \norm{\dpx}_2^2 + \frac{1}{\eta_{t+1}}\Psi(\dpx)\right\},
\end{align*}
where $\Psi(\x) = \Psi(x,b) = 400(\psi(x/b) - 2\nu \ln b)$ is a normal barrier of the conic hull of $\calX$ defined by $con(\X) = \{ \boldsymbol{0}\} \cup \{ (w,b) \mid \frac{w}{b} \in \X, w \in \R^d, b>0\}$, and $\psi$ is a $\nu$-self-concordant barrier of $\calX \subseteq \R^d$, $\dpH_t=\nabla^2\Psi(\dpy_t)+\eta_t(\sigma_{1:t-1}+\lambda_{0:t-1})$, $\dpu_s$ is uniformly sampled from $\mathbb{S}^{d+1}\cap\big(\dpH_s^{-\frac{1}{2}}e_{d+1}\big)^\perp$ and $\dpg_s=d\left(f_s(x_s)+\frac{\lambda_s}{2}\|x_s\|_2^2\right)\dpH_s^{\frac{1}{2}}\dpu_s$ for $s\in[t]$, 

Suppose that the following two conditions hold: (1) the sequence of learning rates $\{\eta_t\}_{t=1}^T$ is non-increasing and satisfies $\frac{1}{\eta_{t+1}}-\frac{1}{\eta_t}\leq C(\lambda_t+\sigma_t)^p$ for some $C>0$ and $p>0$; (2) $\sigma_t\leq\gamma$ holds for some $\gamma>0$ and $\eta_1\leq \frac{1}{32(d+16\sqrt{\nu} C(\gamma+1)^{p})}$, $\lambda_t\in (0,1)$ holds for all $t\in [T]$, and $\lambda_0>0$. Then, we have $\|\dpy_t-\dpy_{t+1}\|_{\dpH_t}\leq\frac{1}{2}$.
\end{lemma}

\begin{proof}
Define the objective of FTRL update to be $\dpF_{t+1}(\x) = \sum_{s=1}^t \dpell_s(\x) + \dpR_{t+1}(\x)$ with $\dpR_{t+1}(\x) = \frac{\lambda_0}{2} \norm{\x}_2^2 + \frac{1}{\eta_{t+1}}\Psi(\x)$ and $\dpell_{s}(\dpx)\triangleq\inner{\dpg_s,\dpx}+\frac{\sigma_s+\lambda_s}{2}\|\dpx-\dpy_s\|_2^2$. Therefore, we have $\dpy_{t+1}=\argmin_{x\in\calX}\dpF_{t+1}(\dpx)$. Define $\Psi_t(\dpx)\triangleq\Psi(\dpx)+\frac{\eta_t\lambda_0}{2}\|\dpx\|_2^2+\eta_t\sum_{s=1}^{t-1}\frac{\sigma_s+\lambda_s}{2}\|\dpx-\dpy_s\|_2^2$. With this definition, we have $\dpH_t=\nabla^2\Psi_t(\dpy_t)$ and $\dpF_{t+1}(\dpx)=\sum_{s=1}^t\inner{\dpg_s,\dpx}+\frac{1}{\eta_{t+1}}\Psi_{t+1}(\dpx)$. Moreover, according to the definition of self-concordant function~(see~\pref{def:SC}), we know that $\Psi_t$ is also a self-concordant function.

Because of the convexity of $\dpF_{t+1}$, in order to prove the desired conclusion, it suffices to show that for any $\dpy'\in \bX$ satisfying $\|\dpy'-\dpy_t\|_{\dpH_t}=\frac{1}{2}$, we have $\dpF_{t+1}(\dpy')\geq \dpF_{t+1}(\dpy_t)$. To this end, we first calculate $\nabla\dpF_{t+1}(\dpy_t)$ and the Hessian of $\dpF_{t+1}$ as follows:
\begin{align}
    \nabla \dpF_{t+1}(\dpy_t) &= \sum_{s=1}^t\nabla\dpell_{s}(\dpy_t) + \lambda_0 \y_t + \frac{1}{\eta_{t+1}}\nabla\Psi(\dpy_t)\nonumber\\
    &= \sum_{s=1}^t\dpg_{s} + \sum_{s=1}^t(\sigma_s+\lambda_s)(\dpy_t-\dpy_s) + \lambda_0 \y_t + \frac{1}{\eta_{t+1}}\nabla\Psi(\dpy_t)\nonumber\\
    &= \nabla\dpF_{t}(\dpy_{t}) + \dpg_t + \left(\frac{1}{\eta_{t+1}}-\frac{1}{\eta_{t}}\right)\nabla\Psi(\dpy_t).\label{eqn: gradient F}
\end{align}
\begin{equation}
 \label{eqn: hessian F}
 \begin{split}
    \nabla^2 \dpF_{t+1}(\dpx) &= \frac{1}{\eta_{t+1}}\nabla^2 \Psi_{t+1}(\dpx) = \frac{1}{\eta_{t+1}}\nabla^2\Psi(\dpx)+\left(\lambda_0 \dpI + \sum_{s=1}^t(\sigma_s+\lambda_s)\dpI\right) \\
    &\succeq  \frac{1}{\eta_t}\nabla^2\Psi(\dpx)+\left(\lambda_0 \dpI + \sum_{s=1}^{t-1}(\sigma_s+\lambda_s)\dpI\right)=\frac{1}{\eta_t}\nabla^2\Psi_t(\dpx),
    \end{split}
\end{equation}
where the inequality is because $\eta_t\geq \eta_{t+1}$. Based on the above, using Taylor's expansion of $\dpF_{t+1}$ at $\dpy_t$, we know that there exists $\dpxi_t$ that lies in the line segment of $\dpy_t$ and $\dpy'$ such that:
\begin{align*}
    & \dpF_{t+1}(\dpy') \\
    & = \dpF_{t+1}(\dpy_t) + \dph^\top \nabla \dpF_{t+1}(\dpy_t) + \frac{1}{2}\|\dph\|_{\nabla^2\dpF_{t+1}(\bxi_t)}^2 \tag{$\dph \triangleq \dpy' - \dpy_t$ and $\bxi_t \in [\dpy_t,\dpy']$}\\
    & \geq \dpF_t(\dpy_t) + \dph^\top \nabla \dpF_t(\dpy_t) + \left(\frac{1}{\eta_{t+1}}-\frac{1}{\eta_{t}}\right)\nabla \Psi(\dpy_t)^\top \dph + \dpg_t^\top\dph+ \frac{1}{2\eta_{t}}\|\dph\|_{\nabla^2\Psi_t(\bxi_t)}^2 \tag{by~\pref{eqn: gradient F} and~\pref{eqn: hessian F}}\\
    & \geq \dpF_t(\dpy_t) + \dph^\top \nabla \dpF_{t}(\dpy_t) + \left(\frac{1}{\eta_{t+1}}-\frac{1}{\eta_{t}}\right)\nabla \Psi(\dpy_t)^\top \dph + \dpg_t^\top\dph+ \frac{1}{2\eta_{t}}\|\dph\|_{\dpH_t}^2\cdot(1-\|\dpy_t-\dpxi_t\|_{\dpH_t})^2 \tag{$\nabla^2\Psi_t(\dpy_t)=\dpH_t$, $\Psi_t$ is a self-concordant function and by~\pref{lemma:shift-norm}}\\
    & \geq \dpF_t(\dpy_t) - \left(\frac{1}{\eta_{t+1}}-\frac{1}{\eta_{t}}\right)\left|\nabla\Psi(\dpy_t)^\top \dph\right| - \|\dpg_t\|_{\dpH_t}^*\|\dph\|_{\dpH_t} + \frac{1}{2\eta_{t}}\|\dph\|_{\dpH_t}^2\cdot(1-\|\dpy_t-\dpxi_t\|_{\dpH_t})^2 \tag{first-order optimality of $\dpy_t$}\\
    & \geq \dpF_t(\dpy_t) - \left(\frac{1}{\eta_{t+1}}-\frac{1}{\eta_{t}}\right)\left|\nabla\Psi(\dpy_t)^\top \dph\right| - \frac{1}{2}\|\dpg_t\|_{\dpH_t}^* + \frac{1}{32\eta_{t}} \tag{$\|\dph\|_{\dpH_t}=\frac{1}{2}$, $\|\dpy_t-\dpxi_t\|_{\dpH_t}\leq\frac{1}{2}$}\\
    & \geq \dpF_t(\dpy_t) - C(\sigma_t+\lambda_t)^p\left|\nabla\Psi(\dpy_t)^\top \dph\right| - \frac{1}{2}\|\dpg_t\|_{\dpH_t}^* + \frac{1}{32\eta_{t}} \tag{by definition of $\eta_t$}\\
    & \geq \dpF_t(\dpy_t) - C(\gamma+1)^{p}\left|\nabla\Psi(\dpy_t)^\top \dph\right| - \frac{1}{2}\|\dpg_t\|_{\dpH_t}^* + \frac{1}{32\eta_1}. \tag{$\sigma_t\leq \gamma$ and $\lambda_t\in(0,1)$, $\{\eta_t\}_{t=1}^T$ is monotonically non-increasing.}
\end{align*}

Furthermore, note that the third term, which is the gradient local norm, can be upper bounded by
\begin{equation}
\label{eq:gradient-norm}
\begin{split}
    \norm{\dpg_t}_{\dpH_t}^{*2} &= d^2 \big(\dpf_t(\x_t) + \frac{\lambda_t}{2} \norm{\x_t}_2^2\big)^2 \big( \lVert\dpH_t^{\frac{1}{2}} \dpu\rVert_{\dpH_t}^*\big)^2\\
    &\leq  d^2 \left( 1+\frac{\lambda_t}{2} \right)^2 \dpu^\T \dpH_t^{\frac{1}{2}} \dpH_t^{-1}\dpH_t^{\frac{1}{2}}\dpu \leq 4d^2.
\end{split}
\end{equation}

For the second term, we have
\begin{align*}
    \left|\nabla\Psi(\dpy_t)^\top \dph\right| &\leq \|\nabla\Psi(\dpy_t)\|_{\nabla^{-2}\Psi(\dpy_t)}\|\dph\|_{\nabla^2\Psi(\dpy_t)} \\
    &\leq \|\nabla\Psi(\dpy_t)\|_{\nabla^{-2}\Psi(\dpy_t)}\|\dph\|_{\nabla^2\Psi_t(\dpy_t)} \tag{$\nabla^2\Psi_t(\dpy)\succeq \nabla^2\Psi(\dpy)$} \\
    &= \|\dpy_t\|_{\nabla^{2}\Psi(\dpy_t)}\|\dph\|_{\nabla^2\Psi_t(\dpy_t)} = \frac{\sqrt{\bar{\nu}}}{2}.
\end{align*}
The last two equations make use of the properties of $\bar{\nu}$-normal barrier (see~\pref{lemma:normal-barrier}): $\nabla^2\Psi(\dpy_t)\dpy_t=-\nabla\Psi(\dpy_t)$ and $\|\dpy_t\|_{\nabla^2\Psi(\dpy_t)}^2=\bar{\nu}$, as well as the fact that $\|\dph\|_{\dpH_t}=\frac{1}{2}$. Moreover, the constructed normal barrier satisfies that $\bar{\nu} = 800 \nu$ (see~\pref{lemma:normal-barrier-nu}). Therefore, we have 
\begin{align*}
    \dpF_{t+1}(\dpy') &\geq \dpF_{t+1}(\dpy_t) - \frac{\sqrt{800 \nu}}{2} C(\gamma+1)^p - d+\frac{1}{32\eta_1} \geq \dpF_{t+1}(\dpy_t),
\end{align*}
where the last step is due to the setting of $\eta_1\leq \frac{1}{32(d+16\sqrt{\nu} C(\gamma+1)^{p})}$. Hence, we complete the proof.
\end{proof}

To apply~\pref{lem: stablity 1/2}, when all the functions are $\beta$-smooth (see~\pref{assumption:smoothness}), we can choose $\gamma=\beta$ to satisfy the condition $\sigma_t\leq \gamma$; when all the functions are $L$-Lipschitz (see~\pref{assumption:lipschitz}), we show in~\pref{lem: upper bound sigma} that choosing $\gamma=4L$ satisfies the condition of $\sigma_t\leq \gamma$.

\subsection{Proof of~\pref{lemma:main-result-smooth-arbitrary}}
\label{appendix:proof-smooth-arbitrary}

To bound the expected regret, we decompose the cumulative regret with respect to $x\in \calX$ in the following way using the functions in the lifted domain defined in \pref{lem: unbiased}, \pref{eqn:lifted_original} and \pref{eqn:lifted_regularize}:

\begin{align}
    & \E\left[\sum_{t=1}^Tf_t(x_t) - \sum_{t=1}^T f_t(x)\right] \nonumber \\
    & = \E\left[\sum_{t=1}^T \dpf_t(\dpx_t) - \sum_{t=1}^T \dpf_t(\dpx)\right]\nonumber\\
    & = \E\left[\sum_{t=1}^T \dpf_t(\dpx_t) - \sum_{t=1}^T \dpf_t(\tilde{\dpx})\right] +  \E\left[\sum_{t=1}^T \dpf_t(\tilde{\dpx}) - \sum_{t=1}^T \dpf_t(\dpx)\right] \nonumber\\
    & = \underbrace{\E\left[\sum_{t=1}^T \f_t(\x_t) - \sum_{t=1}^T \f_t(\y_t)\right]}_{\textsc{Exploration}} + \underbrace{\E\left[\sum_{t=1}^T \f_t(\y_t) - \sum_{t=1}^T \tilde{\f}_t(\y_t)\right]}_{\textsc{Regularization I}} + \underbrace{\E\left[\sum_{t=1}^T \tilde{\f}_t(\y_t) - \sum_{t=1}^T \hat{\f}_t(\y_t)\right]}_{\textsc{Smooth I}}\nonumber\\
    & \quad + \underbrace{\E\left[\sum_{t=1}^T \bfh_t(\y_t) - \sum_{t=1}^T \bfh_t(\tilde{\x})\right]}_{\textsc{Reg Term}} + \underbrace{\E\left[\sum_{t=1}^T \bfh_t(\tilde{\x}) - \sum_{t=1}^T \tilde{\f}_t(\tilde{\x})\right]}_{\textsc{Smooth II}} + \underbrace{\E\left[\sum_{t=1}^T \tilde{\f}_t(\tilde{\x}) - \sum_{t=1}^T \f_t(\tilde{\x})\right]}_{\textsc{Regularize II}}\nonumber\\
    & \quad +\underbrace{\E\left[\sum_{t=1}^T \dpf_t(\tilde{\dpx}) - \sum_{t=1}^T \dpf_t(\dpx)\right]}_{\textsc{Comparator Bias}},\label{eq:decomposition}
\end{align}
where in the second equality, we define $\tilde{\x} \triangleq \left(1-\frac{1}{T}\right)\dpx+\frac{1}{T}\cdot\dpy_1$, where $\dpy_1 = \argmin_{\x \in \bX} \Psi(\x)$. Note that both $\tilde{\x}$ and $\y_1$ belong to the shrunk lifted feasible set $\tilde{\bX}=\{\x=(x,1) \mid x\in \X, \pi_{\dpy_1}(\dpx)\leq 1-\frac{1}{T}\}$. We remind the readers the notations $\tilde{x} = \tilde{\x}_{[1:d]}$ and $y_1 = {\y_1}_{[1:d]}$, and we have $\tilde{\x} = (\tilde{x},1)$ and $\dpy_1=(y_1,1)$.

We now bound the each term of the regret decomposition in~\pref{eq:decomposition} individually. First, for the two terms \textsc{Regularization I} and \textsc{Regularization II}, we have for any $\x \in \bX$,
\begin{align}
    \label{eq:part-1}
    \E\left[\sum_{t=1}^T \dpf_t(\dpx) - \sum_{t=1}^T \wt{\dpf}_t(\dpx)\right]  = \E\left[\sum_{t=1}^T f_t(x) - \sum_{t=1}^T \wt{f}_t(x) \right] \leq \sum_{t=1}^T\frac{\lambda_t}{2},
\end{align}
which essentially is the bias due to introducing the regularization term. 

Second, consider the two terms \textsc{Smooth I} and \textsc{Smooth II}. According to the definition of $\wt{\dpf}_t$ shown in~\pref{lem: unbiased}, we know that $\wt{\dpf}_t$ is $(\beta+\lambda_t)$-smooth. Using the fact that perturbation $\dplus{b}$ has mean $\zero$, we can bound the two term as follows: for any $\x \in \bX$,
\begin{align}
    \label{eq:part-2}
    \E_{\dplus{b}}\left[\sum_{t=1}^T\wt{\dpf}_t\big(\dpx+\dpH_t^{-\frac{1}{2}}\dplus{b}\big)-\sum_{t=1}^T\wt{\dpf}_t(\dpx)\right] &\leq \sum_{t=1}^T\frac{\beta+\lambda_t}{2}\left\|\dpH_t^{-\frac{1}{2}}\dplus{b}\right\|_2^2\nonumber\\ &\leq\sum_{t=1}^T \frac{d(\beta+\lambda_t)}{\sqrt{(\beta+1)(\sigma_{1:t-1}+\lambda_{0:t-1})}},
\end{align}
where the second inequality is because $\dpH_t\succeq\eta_t(\sigma_{1:t-1}+\lambda_{0:t-1}\dpI)$ and $\eta_t=\frac{1}{2d}\sqrt{\frac{\beta+1}{\sigma_{1:t-1}+\lambda_{0:t-1}}+\frac{\nu}{T\log T}}$.

Third, by definition of $\dpy_t$ and the $\beta$-smoothness of function $f_t$, \textsc{Exploration} term can be bounded by
\begin{equation}
    \label{eq:part-3}
    \begin{split}
    \E\left[\sum_{t=1}^T \dpf_t(\dpx_t) - \sum_{t=1}^T \dpf_t(\dpy_t)\right] \leq & \sum_{t=1}^T \frac{\beta}{2}\left\|\dpH_t^{-\frac{1}{2}}\dpu_t\right\|_2^2 \leq \sum_{t=1}^T\frac{d\beta}{\sqrt{(\beta+1)(\sigma_{1:t-1}+\lambda_{0:t-1})}}.
    \end{split}
\end{equation}

Fourth, for \textsc{Comparator Bias}, according to the definition of $\wt{\dpx}$ and using the convexity property of $\dpf_t$, we have
\begin{equation}
    \label{eq:part-compara-4}
    \begin{split}
    \E\left[\sum_{t=1}^T \dpf_t(\wt{\dpx}) - \sum_{t=1}^T \dpf_t(\dpx)\right] \leq & \E\left[\sum_{t=1}^T \dpf_t\left(\frac{1}{T}\dpy_1+\left(1-\frac{1}{T}\right)\dpx\right) - \sum_{t=1}^T \dpf_t(\dpx)\right] \\
    \leq & \E\left[\frac{1}{T}\sum_{t=1}^T \dpf_t(\dpy_1) - \frac{1}{T}\sum_{t=1}^T \dpf_t(\dpx)\right]\leq 2.
    \end{split}
\end{equation}

Therefore, it suffices to further bound the $\textsc{Reg Term}$, which is the expected regret over the smoothed version of the lifted online functions. The following lemma proves the upper bound for the $\textsc{Reg Term}$. We remark that bounding this $\textsc{Reg Term}$ is the most challenging part of the proof and is also the technical reason for us to lift the domain.

\begin{lemma}
\label{lem: reg term smooth}
When loss functions $\{f_t\}_{t=1}^T$ are all $\beta$-smooth, if $T\geq \rho$ (a constant defined in~\pref{alg:lift-smooth}), \pref{alg:lift-smooth} guarantees that 
\begin{align}\label{eq:smooth-regret-upperbound}
\textsc{Reg Term}\leq \otil\left(d\sqrt{\nu T}+\sum_{t=1}^T\frac{d\sqrt{\beta+1}}{\sqrt{\sigma_{1:t-1}+\lambda_{0:t-1}}}\right).    
\end{align}
\end{lemma}
\begin{proof}
According to the definition of \textsc{Reg Term}, we have 
\begin{align}
    \E\left[\sum_{t=1}^T \wh{\dpf}_t(\dpy_t)- \sum_{t=1}^T \wh{\dpf}_t(\tilde{\dpx})\right] & \leq \E\left[\sum_{t=1}^T \left(\nabla \wh{\dpf}_t(\dpy_t)^\top (\dpy_t-\tilde{\dpx})-\frac{\sigma_t+\lambda_t}{2}\|\dpy_t-\tilde{\dpx}\|_2^2\right)\right] \label{eq:derive-1}\\
    & = \E\left[\sum_{t=1}^T\dpg_t^\top(\dpy_t-\tilde{\dpx})-\frac{\sigma_t+\lambda_t}{2}\|\dpy_t-\tilde{\dpx}\|_2^2\right]\label{eq:derive-2}\\
    & = \E\left[\sum_{t=1}^T \dpell_t(\dpy_t) - \sum_{t=1}^T  \dpell_t(\tilde{\dpx}) \right].\label{eq:derive-3}
\end{align}
In above, \pref{eq:derive-1} holds owing to the $(\sigma_t+\lambda_t)$-strong-convexity of $\wh{\dpf}_t$ (actually only in the first $d$ dimension but it is enough as $\dpy_t$ and $\wt{\dpx}$ have the same last coordinate); \pref{eq:derive-2} is true because \pref{lem: unbiased} ensures that $\dpg_t$ is an unbiased estimator of $\nabla\wh{\dpf}_t(\dpy_t)$ in the first $d$ coordinates and meanwhile $\dpy_t-\dpx$ has the last coordinate $0$. The last step shown in~\pref{eq:derive-3} is by introducing the surrogate loss $\dpell_t: \bX \mapsto \R$, defined as $\dpell_t(\dpx)\triangleq \inner{\dpg_t, \dpx}+\frac{\sigma_t+\lambda_t}{2}\|\dpx-\dpy_t\|_2^2.$
Note that according to this construction, we have $\nabla\dpell_t(\dpy_t)=\dpg_t$.

In addition, our FTRL update rule can be written in the following two forms:
\begin{align*}
    \dpy_{t+1} =  & \argmin_{\dpx\in\dpX}\left\{\sum_{s=1}^{t}\left(\inner{\dpg_s, \dpx}+\frac{\sigma_s+\lambda_s}{2}\|\dpy_s-\dpx\|_2^2\right) + \frac{\lambda_0}{2} \norm{\x}_2^2 +\frac{1}{\eta_{t+1}}\Psi(\dpx)\right\} \\
    =  & \argmin_{\dpx\in\dpX}\left\{\sum_{s=1}^{t} \dpell_s(\x) + \frac{\lambda_0}{2} \norm{\x}_2^2 +\frac{1}{\eta_{t+1}}\Psi(\dpx)\right\} \\
    =  & \argmin_{\dpx\in\dpX}\left\{\sum_{s=1}^{t}\inner{\dpg_s, \dpx}+\frac{1}{\eta_{t+1}}\Psi_{t+1}(\dpx)\right\},
\end{align*}
where $\Psi_{t+1}(\dpx) = \Psi(\dpx)+\eta_{t+1}\left(\frac{\lambda_0}{2} \norm{\x}_2^2 + \sum_{s=1}^t\frac{\sigma_s+\lambda_s}{2}\|\dpx - \dpy_s\|_2^2\right)$. As discussed in~\pref{lem: stablity 1/2} $\Psi_{t+1}$ is still a self-concordant function and moreover $\dpH_{t+1}=\nabla^2 \Psi_{t+1}(\dpy_{t+1})$.

Recall the definition $\dpR_{t+1}(\x) = \frac{\lambda_0}{2} \norm{\x}_2^2 + \frac{1}{\eta_{t+1}}\Psi(\x)$ and $\dpF_{t+1}(\x) = \sum_{s=1}^t \dpell_s(\x) + \dpR_{t+1}(\x)$. Denote by $\dpR_{t+1}'(\x) = \frac{\lambda_0}{2} \norm{\x}_2^2 + \frac{1}{\eta_{t+1}}\left(\Psi(\x)-\Psi(\dpy_1)\right)$ the (shifted) regularizer and by $\dpQ_{t+1}(\x) = \sum_{s=1}^t \dpell_s(\x) + \dpR_{t+1}'(\x)$ and (shifted) FTRL objective. Therefore, we have $\dpF_{t+1}(\dpx)=\dpQ_{t+1}(\dpx)+\frac{1}{\eta_{t+1}}\Psi(\dpy_1)$. Then, $\y_{t+1} = \argmin_{\x \in \bX} \dpF_{t+1}(\x)=\argmin_{\x \in \bX} \dpQ_{t+1}(\x)$ according to the FTRL update rule, and we have 
\begin{align*}
    & \sum_{t=1}^T \dpell_t(\y_t) - \sum_{t=1}^T  \dpell_t(\tilde{\x})\\
    & \leq \dpR_{T+1}'(\tilde{\x}) - \dpR_1'(\y_1) + \sum_{t=1}^T \nabla \dpell_t(\y_t)^\T (\y_t - \y_{t+1}) \\
    & \qquad \qquad - \sum_{t=1}^T D_{\dpQ_t + \dpell_t}(\y_{t+1},\y_t) + \sum_{t=1}^T \Big(\dpR_t'(\y_{t+1})- \dpR_{t+1}'(\y_{t+1})\Big)\\
    & \leq \dpR_{T+1}'(\tilde{\x}) - \dpR_1'(\y_1) + \sum_{t=1}^T \dpg_t^\T (\y_t - \y_{t+1}) - \sum_{t=1}^T D_{\dpQ_t + \dpell_t}(\y_{t+1},\y_t)\\
    & \leq \dpR_{T+1}'(\tilde{\x}) - \dpR_1'(\y_1) + \sum_{t=1}^T \dpg_t^\T (\y_t - \y_{t+1}) - \sum_{t=1}^T D_{\dpQ_t}(\y_{t+1},\y_t).
\end{align*}
In above, the first inequality is due to the standard FTRL analysis as shown in \pref{lemma:FTRL-regret}; the second inequality is true because the surrogate loss satisfies that $\nabla\dpell_t(\dpy_t)=\dpg_t$ and $0\leq\dpR_t'(\x) \leq \dpR_{t+1}'(\x)$ holds for any $\x \in \bX$ as the learning rate is monotonically non-increasing and $\dpy_1=\argmin_{\dpx\in\calX}\Psi(\dpx)$. The last inequality follows from $\nabla^2 \dpell_t(x) = (\sigma_t + \lambda_t) \boldsymbol{I}$ and the following inequality:
\begin{align*}
    D_{\dpQ_t + \dpell_t}(\y_{t+1},\y_t) &= D_{\dpQ_t}(\y_{t+1},\y_t) + D_{\dpell_t}(\y_{t+1},\y_t) \\
    &= D_{\dpQ_t}(\y_{t+1},\y_t) + \frac{\sigma_t + \lambda_t}{2} \norm{\y_{t+1} - \y_t}_2^2 \geq D_{\dpQ_t}(\y_{t+1},\y_t).
\end{align*}
In addition, by Taylor expansion, we know that $D_{\dpQ_t}(\y_{t+1},\y_t) = \frac{1}{2} \norm{\y_{t+1} - \y_t}_{\nabla^2 \dpQ_t(\boldsymbol{\xi}_t)}^2$ for some $\bxi_t \in [\y_t, \y_{t+1}]$, and $\nabla^2 \dpQ_t(\x) = \nabla^2 \dpF_t(\x) = \frac{1}{\eta_t} \nabla^2 \Psi_t(\x)$ as shown in the first equality of~\pref{eqn: hessian F}. Therefore, combining all above, we get that
\begin{equation}
    \label{eq:upper-bound-surrogate-loss}
    \begin{split}
     & \sum_{t=1}^T \dpell_t(\y_t) - \sum_{t=1}^T  \dpell_t(\tilde{\x}) \\
     & \leq \frac{\lambda_0}{2} \norm{\tilde{\x}}_2^2 + \frac{\Psi(\tilde{\x})-\Psi(\y_1)}{\eta_{T+1}} + \sum_{t=1}^T \dpg_t^\T (\y_t - \y_{t+1}) - \sum_{t=1}^T \frac{1}{2 \eta_t} \norm{\y_{t+1} - \y_t}_{\nabla^2 \Psi_t(\bxi_t)}^2.
    \end{split}
\end{equation}
    
In the following, we proceed to analyze the crucial terms $\dpg_t^\T (\y_t - \y_{t+1})$ and $\norm{\y_{t+1} - \y_t}_{\nabla^2 \Psi_t(\bxi_t)}^2$. For the first term, by Holder's inequality, we have
\begin{equation}
    \label{eq:inner-holder}
    \dpg_t^\T (\y_t - \y_{t+1}) \leq \|\dpg_t\|_{\dpH_t}^*\cdot \|\dpy_t-\dpy_{t+1}\|_{\dpH_t}.
\end{equation}
The second term is more involved to analyze. To do this, we first verify that the conditions required in~\pref{lem: stablity 1/2} are indeed satisfied. First, it is direct to see that $\{\eta_{t}\}_{t=1}^T$ is non-increasing and 
\begin{align*}
    \frac{1}{\eta_{t+1}}-\frac{1}{\eta_t} &\leq 2d\left(\frac{1}{\sqrt{\frac{\beta+1}{\lambda_{0:t}+\sigma_{1:t}}+\frac{\nu}{T\log T}}}-\frac{1}{\sqrt{\frac{\beta+1}{\lambda_{0:t-1}+\sigma_{1:t-1}}+\frac{\nu}{T\log T}}}\right)\\
    &\leq \frac{2d}{\sqrt{\beta+1}}\left(\sqrt{\lambda_{0:t}+\sigma_{1:t}}-\sqrt{\lambda_{0:t-1}+\sigma_{1:t-1}}\right) \\
    &\leq \frac{2d\sqrt{\lambda_t+\sigma_t}}{\sqrt{\beta+1}},
\end{align*}
where the second inequality is because $(a+c)^{-\nicefrac{1}{2}}-(b+c)^{-\nicefrac{1}{2}}$ is decreasing in $c$ when $a\leq b$. Therefore, this satisfies that $\eta_{t+1}^{-1}-\eta_t^{-1}\leq C(\lambda_t+\sigma_t)^p$ with $C=2d(\beta+1)^{-\nicefrac{1}{2}}$ and $p=\frac{1}{2}$. In addition, note that as $T\geq \rho= 512\nu(1+32\sqrt{\nu})^2$ and $\lambda_0\geq (\beta+1)\rho\nu^{-1}$, we have 
\begin{align*}
    \eta_1=\frac{1}{2d}\sqrt{\frac{\beta+1}{\lambda_0}+\frac{\nu}{T\log T}}\leq\frac{1}{2d}\sqrt{\frac{2\nu}{\rho}}=\frac{1}{32d(1+32\sqrt{\nu})}= \frac{1}{32(d+16\sqrt{\nu}C(\gamma+1)^p)},
\end{align*}
with $\gamma=\beta$. Therefore, according to~\pref{lem: stablity 1/2}, we show that $\|\dpy_t-\dpy_{t+1}\|_{\dpH_t}\leq\frac{1}{2}$. Then, due to the nice properties of optimization with self-concordant functions (see \pref{lemma:shift-norm}), we obtain that 
\begin{equation}
    \label{eq:stability-term}
    \|\dpy_{t+1}-\dpy_t\|_{\nabla^2\Psi_{t}(\bxi_t)}\geq \|\dpy_{t+1}-\dpy_t\|_{\nabla^2\Psi_{t}(\y_t)}\cdot(1-\|\dpy_{t+1}-\bxi_t\|_{\nabla^2\Psi_{t}(\y_t)}) \geq \frac{1}{2} \|\dpy_{t+1}-\dpy_t\|_{\dpH_t},
\end{equation}
where the last inequality makes use of the result $\|\dpy_t-\dpy_{t+1}\|_{\dpH_t}\leq\frac{1}{2}$ as well as the fact that $\nabla^2\Psi_{t}(\y_t) = \dpH_t$. Plugging inequalities \pref{eq:inner-holder} and \pref{eq:stability-term} to the regret upper bound achieves
\begin{align}
    & \sum_{t=1}^T \dpell_t(\y_t) - \sum_{t=1}^T  \dpell_t(\tilde{\x}) \nonumber \\
    & \leq \frac{\lambda_0}{2} \norm{\tilde{\x}}_2^2 + \frac{\Psi(\tilde{\x})-\Psi(\y_1)}{\eta_{T+1}} + \sum_{t=1}^T \dpg_t^\T (\y_t - \y_{t+1}) - \sum_{t=1}^T \frac{1}{2 \eta_t} \norm{\y_{t+1} - \y_t}_{\nabla^2 \Psi_t(\bxi_t)}^2\nonumber \\
    & \leq \frac{\lambda_0}{2} \norm{\tilde{\x}}_2^2 + \frac{\Psi(\tilde{\x})-\Psi(\y_1)}{\eta_{T+1}} + \sum_{t=1}^T \left( \|\dpg_t\|_{\dpH_t}^*\cdot \|\dpy_t-\dpy_{t+1}\|_{\dpH_t}-\frac{1}{8\eta_t}\|\dpy_{t+1}-\dpy_t\|_{\dpH_t}^2\right)\nonumber\\
    & \leq \frac{\lambda_0}{2} \norm{\tilde{\x}}_2^2 + \frac{\Psi(\tilde{\x})-\Psi(\y_1)}{\eta_{T+1}} + \sum_{t=1}^T 2\eta_t\|\dpg_t\|_{\dpH_t}^{*2} \label{eqn: regret-form} \\
    & \leq \O(1) + \frac{\Psi(\tilde{\x})-\Psi(\y_1)}{\eta_{T+1}} +\sum_{t=1}^T8\eta_t d^2,\tag{\pref{eq:gradient-norm}} \nonumber\\
    & \leq \O(1) + \frac{\Psi(\tilde{\x})-\Psi(\y_1)}{\eta_{T+1}} +\sum_{t=1}^T\frac{8d\sqrt{\beta+1}}{\sqrt{\sigma_{1:{t-1}}+\lambda_{0:t-1}}} +\O\left(d\sqrt{\nu T\log T}\right),\nonumber\\
    & \leq \O(1) + \O\left(d \sqrt{\nu T \log T}\right) +\sum_{t=1}^T\frac{8d\sqrt{\beta+1}}{\sqrt{\sigma_{1:t-1}+\lambda_{0:t-1}}},\label{eq:part-4}
\end{align}
where \pref{eq:part-4} holds because of the following three facts. First, as both $\tilde{\x}$ and $\y_1$ belong to the shrunk lifted domain $\tilde{\bX}=\{\dpx\;|\;\pi_{\dpy_1}(\dpx)\leq 1-\frac{1}{T}\}$, based on \pref{lemma:bound-Minkowski}, we have that $0 \leq \Psi(\dpx)-\Psi(\dpy_1) \leq \bar{\nu}\ln \big(\frac{1}{1-\pi_{\dpy_1}(\dpx)}\big)\leq \bar{\nu} \log T$ holds for all $\dpx \in \wt{\dpX}$. Second, as demonstrated in \pref{lemma:normal-barrier-nu}, the normal barrier we choose in~\pref{alg:lift-smooth} ensures that $\bar{\nu} = 800 \nu = \O(\nu)$. Third, $\eta_{T+1}\geq \frac{1}{2d}\sqrt{\frac{\nu}{T\log T}}$. This finishes the proof of~\pref{lem: reg term smooth}.
\end{proof}
Now we are ready to prove our main lemma (\pref{lemma:main-result-smooth-arbitrary}). Below we restate the lemma for convenience.
\begin{lemma}
\label{lemma:main-result-smooth-arbitrary-appendix}
With any regularization coefficients $\{\lambda_t\}_{t=1}^T \in (0,1)$,
\pref{alg:lift-smooth} guarantees:
\begin{equation}
\label{eqn: strongly-theorem-arbitrary-appendix}
    \Reg = \Ot\left(d \sqrt{\nu T}+ \lambda_{1:T-1} + \sum_{t=1}^{T-1}\frac{d\sqrt{\beta+1}}{\sqrt{\sigma_{1:t} + \lambda_{0:t}}}\right),
\end{equation}
if loss functions $\{f_t\}_{t=1}^T$ are all $\beta$-smooth and $T\geq \rho$ (a constant defined in \pref{alg:lift-smooth}).
\end{lemma}
\begin{proof}
Combining all the above terms in~\pref{eq:part-1},~\pref{eq:part-2},~\pref{eq:part-3},~\pref{eq:part-compara-4}, and~\pref{eq:smooth-regret-upperbound} as well as the decomposition in~\pref{eq:decomposition}, we obtain the following expected regret upper bound:
\begin{align}
     &\E\left[\sum_{t=1}^Tf_t(x_t) - \sum_{t=1}^T f_t(x)\right] \nonumber\\
    &\overset{\eqref{eq:decomposition}}{\leq} \E\left[\sum_{t=1}^T \bfh_t(\y_t) - \sum_{t=1}^T \bfh_t(\tilde{\x})\right] + 2\nonumber\\
     & \qquad + \E\left[\sum_{t=1}^T \f_t(\x_t) - \sum_{t=1}^T \f_t(\y_t)\right] + \E\left[\sum_{t=1}^T \f_t(\y_t) - \sum_{t=1}^T \tilde{\f}_t(\y_t)\right] + \E\left[\sum_{t=1}^T \tilde{\f}_t(\y_t) - \sum_{t=1}^T \hat{\f}_t(\y_t)\right]\nonumber\\
     & \qquad + \E\left[\sum_{t=1}^T \bfh_t(\tilde{\x}) - \sum_{t=1}^T \tilde{\f}_t(\tilde{\x})\right] + \E\left[\sum_{t=1}^T \tilde{\f}_t(\tilde{\x}) - \sum_{t=1}^T \f_t(\tilde{\x})\right] \nonumber\\
    &\leq  \O\left(\lambda_{1:T}+d \sqrt{\nu T\log T} + \sum_{t=1}^T \frac{d (\beta+\lambda_t)}{\sqrt{\beta+1}\sqrt{\sigma_{1:t-1}+\lambda_{0:t-1}}}+\sum_{t=1}^T \frac{d\sqrt{\beta+1}}{\sqrt{\sigma_{1:t-1}+\lambda_{0:t-1}}}\right)\nonumber\\
    &\leq  \Ot\left(\lambda_{1:T} + d \sqrt{\nu T}+ \sum_{t=1}^T\frac{d\sqrt{\beta+1}}{\sqrt{\sigma_{1:t-1} + \lambda_{0:t-1}}}\right) \tag{$\lambda_t\in(0,1)$}\nonumber\\
    &\leq \Ot\left(\lambda_{1:T-1} + d \sqrt{\nu T}+ \sum_{t=1}^{T-1}\frac{d\sqrt{\beta+1}}{\sqrt{\sigma_{1:t} + \lambda_{0:t}}}\right), \label{eq:derive-last-step}
\end{align}
where the last step hold because our choice of regularization coefficients $\lambda_t \in (0,1)$ for $t\in[T]$ and the input parameter $\lambda_0 \geq 1$, which finishes the proof.
\end{proof}

\subsection{Proof of~\pref{lemma:H-function-main}}
\label{appendix:H-function-main}
\begin{proof}
We prove the claim~\pref{eq:near-optimal-H-function} by induction, whose proof technique is similar to~\citep[Lemma 3.1]{NIPS'07:AOGD}. 

Consider the base case when $t=1$. For simplicity, we define $\lambda^*_0=\lambda_0$. If $\lambda_1\leq \lambda_1^*\triangleq\argmin_{\lambda_1\geq 0}\calB(\lambda_1)$, we have $\calB(\lambda_1)=\lambda_1+\frac{d\sqrt{\beta+1}}{\sqrt{\sigma_1+\lambda_{0:1}}}=2\lambda_1 \leq 2\lambda^*_1 \leq 2\calB(\lambda^*_1)$, where the second equality is true because of the condition in~\pref{eqn:oracle-tuning-strongly-main}. Otherwise, we have $\calB(\lambda_1)=\frac{2d\sqrt{\beta+1}}{\sqrt{\sigma_1+\lambda_{0:1}}}\leq\frac{2d\sqrt{\beta+1}}{\sqrt{\sigma_1+\lambda_{0:1}^*}}\leq 2\calB(\lambda_1^*)$. Combining both scenarios verifies the base case.

Suppose we have $\calB(\{\lambda_s\}_{s=1}^{t-1})\leq 2\min_{\{\lambda_s'\}_{s=1}^{t-1}\geq 0}\calB(\{\lambda_s'\}_{s=1}^{t-1})$. With a slight abuse of notation, we set $\{\lambda_s^*\}_{s=1}^t=\argmin_{\{\lambda_s'\}_{s=1}^t\geq 0}\calB(\{\lambda_s'\}_{s=1}^t)$. Similarly, if $\lambda_{1:t}\leq \lambda_{1:t}^*$, we have
\begin{align*}
    \calB(\{\lambda_{s}\}_{s=1}^t\}) &= \lambda_{1:t}+\sum_{s=1}^t \frac{d\sqrt{\beta+1}}{\sqrt{\sigma_{1:s}+\lambda_{0:s}}} = \lambda_{1:t}+\sum_{s=1}^t \lambda_{s}= 2\lambda_{1:t}\leq 2\lambda_{1:t}^*\leq 2\calB(\{\lambda_{s}^*\}_{s=1}^t).
\end{align*}
Otherwise, we have
\begin{align*}
    &\lambda_t+\frac{d\sqrt{\beta+1}}{\sqrt{\sigma_{1:t}+\lambda_{0:t}}} = \frac{2d\sqrt{\beta+1}}{\sqrt{\sigma_{1:t}+\lambda_{0:t}}}\leq \frac{2d\sqrt{\beta+1}}{\sqrt{\sigma_{1:t}+{\lambda}_{0:t}^*}}\leq 2\left(\lambda_t^*+\frac{d\sqrt{\beta+1}}{\sqrt{\sigma_{1:t}+{\lambda}_{1:t}^*}}\right).
\end{align*}
Using the induction hypothesis, we have
\begin{align*}
    \calB(\{\lambda_{s}\}_{s=1}^t\}) &= \calB(\{\lambda_{s}\}_{s=1}^{t-1}\}) + \lambda_t+\frac{d\sqrt{\beta+1}}{\sqrt{\sigma_{1:t}+\lambda_{0:t}}} \\ 
    &\leq \min_{\{\lambda_{s}'\}_{s=1}^{t-1}\geq 0}2\calB(\{\lambda_s'\}_{s=1}^{t-1})+2\left(\lambda_t^*+\frac{d\sqrt{\beta+1}}{\sqrt{\sigma_{1:t}+{\lambda}_{1:t}^*}}\right) \\
    &\leq 2\calB(\{\lambda_s^*\}_{s=1}^{t-1})+2\left(\lambda_t^*+\frac{d\sqrt{\beta+1}}{\sqrt{\sigma_{1:t}+{\lambda}_{1:t}^*}}\right)\\
    &=2\calB(\{\lambda_{s}^*\}_{s=1}^t),
\end{align*}
where the first inequality is because of the induction hypothesis. Combining both cases, we have that $\calB(\{\lambda_{s}\}_{s=1}^t\})\leq 2\calB(\{\lambda_{s}^*\}_{s=1}^t)$.
\end{proof}

\subsection{Proof of \pref{thm:main-result-smooth}}
\label{appendix:proof-main-result-smooth}
\begin{proof}
By~\pref{lemma:main-result-smooth-arbitrary} we have
\[
    \Reg \leq \Ot\left( \lambda_{1:T-1} + d\sqrt{\nu T} + \sum_{t=1}^{T-1}\frac{d\sqrt{\beta+1}}{\sqrt{\sigma_{1:t} + \lambda_{0:t}}} \right),
\]
which holds for any sequence of regularization coefficients $\lambda_1,\ldots,\lambda_T \in (0,1)$. Moreover, due to the specific calculation of regularization coefficients (see~\pref{eqn:oracle-tuning-strongly-main}) and from~\pref{lemma:H-function-main}, we immediately achieve for any $\lambda_{1}^*,\ldots,\lambda_T^*\geq 0$,  
\begin{align*}
    \Reg &\leq \Ot\left( \lambda_{1:T-1}^* + d\sqrt{\nu T} + \sum_{t=1}^{T-1}\frac{d\sqrt{\beta+1}}{\sqrt{\sigma_{1:t} + \lambda_{0:t}^*}} \right),
\end{align*}
which finishes the proof of~\pref{thm:main-result-smooth}.
\end{proof}


\subsection{Proofs for Implications of~\pref{thm:main-result-smooth}}
\label{appendix: smooth implication proof}
In this section, we provide the proofs of implications in~\pref{sec: implication}.

~\\

\begin{proof}[of~\pref{cor: only-smooth}]
Since \pref{thm:main-result-smooth} holds for any non-negative sequence of $\{\lambda_t^*\}_{t=1}^T$, in particular, we choose $\lambda_1^*=(1+\beta)^{\frac{1}{3}}d^{\frac{2}{3}}T^{\frac{2}{3}}$ and $\lambda_t^*=0$ for all $t \geq 2$, then with $\nu=\order(d)$, we obtain that
\begin{align*}
\Reg \overset{\eqref{eqn: strongly-theorem}}{\leq} \Ot\left((1+\beta)^{\frac{1}{3}}d^{\frac{2}{3}} T^{\frac{2}{3}} + d^{\frac{3}{2}}\sqrt{T}+ (1+\beta)^{\frac{1}{3}}d^{\frac{2}{3}} T^{\frac{2}{3}}  \right) = \Ot\left(d^{\frac{3}{2}}\sqrt{T}+(1+\beta)^{\frac{1}{3}}d^{\frac{2}{3}} T^{\frac{2}{3}}\right),
\end{align*}
where the last step holds due to $\nu = \O(d)$ (see \pref{lemma:exist-SCB}).
\end{proof}

\begin{proof}[of~\pref{cor: strongly-convex-smooth}]
Since \pref{thm:main-result-smooth} holds for any non-negative sequence of $\{\lambda_t^*\}_{t=1}^T$, in particular, we choose $\lambda_t^*=0$ for all $t \geq 1$ and, then with $\nu=\order(d)$, we obtain that
\begin{align*}
    \Reg \overset{\eqref{eqn: strongly-theorem}}{\leq} \Ot\left(d\sqrt{\nu T} +  d\sqrt{\frac{(1+\beta) T}{\sigma}} \right) =  \Ot\left(d^{\frac{3}{2}}\sqrt{T}+d\sqrt{\frac{T(1+\beta)}{\sigma}}\right),
\end{align*}
which ends the proof.
\end{proof}

\begin{proof}[of~\pref{cor: worst-best-case}]
In the first environment where there are $M$ rounds in which the loss function is $0$-strongly convex, to make the right hand side of~\pref{eqn: strongly-theorem} the largest, we have $\sigma_{\tau}=0$ when $\tau\in[M]$ and $\sigma_{\tau}=\sigma$ when $\tau>M$. Set $\lambda_t^*=0$ for all $t\geq 2$. According to \pref{eqn: strongly-theorem} shown in \pref{thm:main-result-smooth} and the choice $\nu=\order(d)$, we have
\begin{align*}
    \Reg & \leq \otil\left(\lambda_1^*+d\sqrt{\nu T}+\sum_{t=1}^{T-1}\frac{d\sqrt{1+\beta}}{\sqrt{\lambda_1^*+\sigma_{1:t}}}\right) \\
    & \leq \otil\left(\lambda_1^*+d^{\frac{3}{2}}\sqrt{T}+\frac{d\sqrt{1+\beta}M}{\sqrt{\lambda_1^*}}+d\sqrt{1+\beta}\min\left\{\frac{(T-M)}{\sqrt{\lambda_1^*}},\sqrt{\frac{T-M}{\sigma}}\right\}\right) \\
    & \leq \otil \left(d^{\frac{2}{3}}(1+\beta)^{\frac{1}{3}} M^{\frac{2}{3}}+d^{\frac{3}{2}}\sqrt{T}+d\sqrt{1+\beta}\min\left\{\frac{T-M}{(1+\beta)^{\frac{1}{6}}d^{\frac{1}{3}}M^{\frac{1}{3}}}, \sqrt{\frac{T-M}{\sigma}}\right\}\right), \tag{choosing $\lambda_1^*=(1+\beta)^{\frac{1}{3}}d^{\frac{2}{3}}M^{\frac{2}{3}}$}
\end{align*}
which leads to the first regret bound. Next, we consider the second environment where the first 
$T-M$ loss functions are $\sigma$-strongly convex. Similarly, we choose $\lambda_t^*=0$ for $t\geq 2$ and we have our regret bounded as follows:
\begin{align*}
    \Reg & \leq \otil\left(\lambda_1^*+d\sqrt{\nu T}+\sum_{t=1}^{T-1}\frac{d\sqrt{1+\beta}}{\sqrt{\lambda_1^*+\sigma_{1:t}}}\right) \\
    & \leq \otil\left(\lambda_1^*+d^{\frac{3}{2}}\sqrt{T}+\frac{dM\sqrt{1+\beta}}{\sqrt{(T-M)\sigma+\lambda_1^*}}+d\sqrt{1+\beta}\min\left\{\frac{T-M}{\sqrt{\lambda_1^*}},\sqrt{\frac{T-M}{\sigma}}\right\}\right).
\end{align*}
When $T-M=\Theta(T)$, we have
\begin{align*}
    \Reg&\leq \otil\left(\lambda_1^*+d^{\frac{3}{2}}\sqrt{T}+\frac{dM\sqrt{1+\beta}}{\sqrt{T\sigma+\lambda_1^*}}+d\sqrt{1+\beta}\sqrt{\frac{T}{\sigma}}\right) \\
    &\leq \otil\left(\lambda_1^*+d^{\frac{3}{2}}\sqrt{T}+\frac{dM\sqrt{1+\beta}}{\sqrt{T\sigma}}+d\sqrt{1+\beta}\sqrt{\frac{T}{\sigma}}\right)\\
    &\leq \otil\left(d^{\frac{3}{2}}\sqrt{T}+\frac{dT\sqrt{1+\beta}}{\sqrt{\sigma(T-M)}}\right),
\end{align*}
where the last inequality is by choosing $\lambda_1^*=0$.
When $T-M=o(T)$, we have $M=\Theta(T)$. Furthermore, if $\lambda_1^*\leq \sigma(T-M)$, we have $\frac{T-M}{\sqrt{\lambda_1^*}}\geq \sqrt{\frac{T-M}{\sigma}}$ and therefore,
\begin{align*}
    \Reg&\leq \otil\left(\lambda_1^*+d^{\frac{3}{2}}\sqrt{T}+\frac{dT\sqrt{1+\beta}}{\sqrt{(T-M)\sigma}}+d\sqrt{1+\beta}\sqrt{\frac{T-M}{\sigma}}\right)\\
    &\leq \otil\left(d^{\frac{3}{2}}\sqrt{T}+\frac{dT\sqrt{\beta+1}}{\sqrt{\sigma(T-M)}}\right),
\end{align*}
where the last inequality is by choosing $\lambda_1^*=0$. On the other hand, if $\lambda_1^*\geq \sigma(T-M)$, we have
\begin{align*}
    \Reg\leq\otil\left(\lambda_1^*+d^{\frac{3}{2}}\sqrt{T}+\frac{dT\sqrt{1+\beta}}{\sqrt{\lambda_1^*}}\right)\leq \otil\left(\sigma(T-M)+d^{\frac{2}{3}}(1+\beta)^{\frac{1}{3}}T^{\frac{2}{3}}+d^{\frac{3}{2}}\sqrt{T}\right),
\end{align*}
where the last inequality is by choosing $\lambda_1^*=\max\left\{\sigma(T-M), d^{\frac{2}{3}}(1+\beta)^{\frac{1}{3}}T^{\frac{2}{3}}\right\}$. Combining the above bounds, we have
\begin{align*}
    \Reg\leq \otil\left(d^{\frac{3}{2}}\sqrt{T}+\min\left\{\frac{dT\sqrt{1+\beta}}{\sqrt{\sigma(T-M)}}, \sigma(T-M)+d^{\frac{2}{3}}(1+\beta)^{\frac{1}{3}}T^{\frac{2}{3}}\right\}\right),
\end{align*}
which finishes the proof.
\end{proof}

\begin{proof}[of~\pref{cor: intermediate-case}]
Since \pref{thm:main-result-smooth} holds for any sequence of $\{\lambda_t^*\}_{t=1}^T$, in particular, we choose $\lambda_t^*=0$ for all $t\geq 2$ and set $\lambda_1^*=(1+\beta)^{\mu_0}d^{\mu_1}\cdot T^{\mu_2}$ with $\mu_2 < 1$, then we obtain that
\begin{align*}
    \Reg &\leq \Ot\left( \lambda_1^*+d\sqrt{\nu T}+d\sqrt{\beta+1}\sum_{t=1}^{T-1}\frac{1}{\sqrt{t^{1-\alpha}+\lambda_1^*}}\right) \\
    &\leq \otil\left(\lambda_1^*  + d^{\frac{3}{2}}\sqrt{T}+d\sqrt{\beta+1}\min\left\{T^{\frac{1}{2}+\frac{\alpha}{2}}, \frac{T}{\sqrt{\lambda_1^*}}\right\}\right) \\
    &=\otil\left((1+\beta)^{\mu_0}d^{\mu_1}T^{\mu_2} + d^{\frac{3}{2}}\sqrt{T} + d\sqrt{\beta+1}\min\left\{T^{\frac{1+\alpha}{2}}, (1+\beta)^{-\frac{\mu_0}{2}}d^{-
    \frac{\mu_1}{2}}T^{1-\frac{\mu_2}{2}}\right\}\right).
\end{align*}
First, the above bound can be upper bounded by
\begin{align*}
    \Reg \leq \otil\left(d^{\frac{3}{2}}\sqrt{T}+(1+\beta)^{\mu_0}d^{\mu_1}T^{\mu_2}+(1+\beta)^{\frac{1-\mu_0}{2}}d^{1-\frac{\mu_1}{2}}T^{1-\frac{\mu_2}{2}}\right) \leq \otil(d^{\frac{3}{2}}\sqrt{T}+(1+\beta)^{\frac{1}{3}}d^{\frac{2}{3}}T^{\frac{2}{3}}),
\end{align*}
where the last inequality is by choosing $\mu_0=\frac{1}{3}$, $\mu_1=\frac{2}{3}$ and $\mu=\frac{2}{3}$.
Furthermore, when $\alpha\in [0, \frac{1}{3}-\frac{1}{3}\log_T(\beta+1)-\frac{2}{3}\log_Td]$, we have $T^{\frac{1+\alpha}{2}}\leq(1+\beta)^{-\frac{1}{6}} d^{-\frac{1}{3}}T^{\frac{2}{3}}$. Therefore, set $\mu_0=\mu_1=\mu_2=0$ and we have
\begin{align*}
    \Reg &\leq \otil\left(d^{\frac{3}{2}}\sqrt{T}+d\sqrt{1+\beta}T^{\frac{1+\alpha}{2}}\right).
\end{align*}
Combining both situations finishes the proof.
\end{proof}
\section{Omitted Details for \pref{sec:lipschitz-BCO}}
\label{appendix: lip-bco-proof}

In this section, we show the proof in the Lipschitz BCO setting. Specifically, we show the proof for the main theorem of Lipschitz BCO in~\pref{appendix: lip-bco-main-thm-proof} and show the proofs for the implications of~\pref{thm:main-result-lipschitz} in~\pref{appendix: lip implication proof}.

\subsection{Proof of~\pref{thm:main-result-lipschitz}}
\label{appendix: lip-bco-main-thm-proof}
Following the same regret decomposition as~\pref{eq:decomposition}, we decompose the regret into the following terms where $\wt{\dpx}$ is defined the same as the one in~\pref{eq:decomposition}.
\begin{align}
    & \E\left[\sum_{t=1}^Tf_t(x_t) - \sum_{t=1}^T f_t(x)\right] \nonumber\\ 
    & = \underbrace{\E\left[\sum_{t=1}^T \f_t(\x_t) - \sum_{t=1}^T \f_t(\y_t)\right]}_{\textsc{Exploration}} + \underbrace{\E\left[\sum_{t=1}^T \f_t(\y_t) - \sum_{t=1}^T \tilde{\f}_t(\y_t)\right]}_{\textsc{Regularization I}} + \underbrace{\E\left[\sum_{t=1}^T \tilde{\f}_t(\y_t) - \sum_{t=1}^T \hat{\f}_t(\y_t)\right]}_{\textsc{Smooth I}}\nonumber\\
    & \quad + \underbrace{\E\left[\sum_{t=1}^T \bfh_t(\y_t) - \sum_{t=1}^T \bfh_t(\tilde{\x})\right]}_{\textsc{Reg Term}} + \underbrace{\E\left[\sum_{t=1}^T \bfh_t(\tilde{\x}) - \sum_{t=1}^T \tilde{\f}_t(\tilde{\x})\right]}_{\textsc{Smooth II}} + \underbrace{\E\left[\sum_{t=1}^T \tilde{\f}_t(\tilde{\x}) - \sum_{t=1}^T \f_t(\tilde{\x})\right]}_{\textsc{Regularization II}}\nonumber\\
    & \quad +\underbrace{\E\left[\sum_{t=1}^T \dpf_t(\tilde{\dpx}) - \sum_{t=1}^T \dpf_t(\dpx)\right]}_{\textsc{Comparator Bias}},\label{eq:decomposition-lip}
\end{align}

For terms \textsc{Regularization I} and \textsc{Regularization II}, we bound them in the same way as shown in~\pref{eq:part-1}: for any $\dpx\in \dpX$,
\begin{align}\label{eq:part-1-lip}
    \mathbb{E}\left[\sum_{t=1}^T \dpf_t(\dpx) - \sum_{t=1}^T \wt{\dpf}_t(\dpx) \right] \leq \sum_{t=1}^T\frac{\lambda_t}{2}.
\end{align}

For terms \textsc{Smooth I} and \textsc{Smooth II}, instead of using the smoothness property in~\pref{appendix:details-smooth}, we use the Lipschitzness of $\wt{\dpf}_t$ and bound the two terms as follows:
\begin{align}\label{eq:part-2-lip}
    &\mathbb{E}_{\dplus{b}}\left[\sum_{t=1}^T\wt{\dpf}_t(\dpx+\dpH_t^{-\frac{1}{2}}\dplus{b})-\sum_{t=1}^T\wt{\dpf}_t(\dpx)\right] \nonumber \\
    &\leq \sum_{t=1}^T\left(L+\lambda_t\right)\left\|\dpH_t^{-\frac{1}{2}}\dplus{b}\right\|_2\leq \sum_{t=1}^T\frac{L+1}{\sqrt{\eta_t(\sigma_{1:t-1}+\lambda_{0:t-1})}}\leq \sum_{t=1}^T\frac{d^{\frac{2}{3}}(L+1)^{\frac{2}{3}}}{(\sigma_{1:t-1}+\lambda_{0:t-1})^{\frac{1}{3}}},
\end{align}
where the last inequality is by the definition of $\eta_t\geq d^{-\frac{4}{3}}(L+1)^{\frac{2}{3}}(\sigma_{1:t-1}+\lambda_{0:t-1})^{-\frac{1}{3}}$.

For term \textsc{Exploration}, we again use the Lipschitzness of $\dpf_t$ and have
\begin{align}\label{eq:part-3-lip}
    &\mathbb{E}\left[\sum_{t=1}^T \dpf_t(\dpx_t)-\sum_{t=1}^T\dpf_t(\dpy_t)\right]\nonumber\\
    &\leq\mathbb{E}\left[\sum_{t=1}^TL\|\dpx_t-\dpy_t\|_2\right] \leq\sum_{t=1}^TL\left\|\dpH_t^{-\frac{1}{2}}\dpu_t\right\|_2\leq \frac{d^{\frac{2}{3}}(L+1)^{\frac{2}{3}}}{(\sigma_{1:t-1}+\lambda_{0:t-1})^{\frac{1}{3}}}.
\end{align}

For term \textsc{Comparator Bias}, as shown in~\pref{eq:part-compara-4}, we have
\begin{align}\label{eq:part-4-lip}
    \textsc{Comparator Bias}\leq 2.
\end{align}
Next, we show the following lemma bounding~\textsc{Reg Term}.

\begin{lemma}
\label{lem: reg term Lipschitz}
When loss functions $\{f_t\}_{t=1}^T$ are all $L$-Lipschitz, if $T\geq \rho'$ (a constant defined in~\pref{alg:lift-Lipschitz}), ~\pref{alg:lift-Lipschitz} guarantees that 
\begin{align}\label{eqn:reg-lip}
    \textsc{Reg Term}\leq \otil\left(d^{\frac{4}{3}}\nu T^{\frac{1}{3}}+\sum_{t=1}^T\frac{d^{\frac{2}{3}}(L+1)^{\frac{2}{3}}}{(\sigma_{1:t-1}+\lambda_{0:t-1})^{\frac{1}{3}}}\right).
\end{align}
\end{lemma}

\begin{proof}
Similar to the analysis in~\pref{lem: reg term smooth}, we first verify the conditions in~\pref{lem: stablity 1/2} are satisfied. It is direct to see that $\{\eta_{t}\}_{t=1}^T$ is non-increasing and
\begin{align*}
    \frac{1}{\eta_t}-\frac{1}{\eta_{t+1}} 
&\leq d^{\frac{4}{3}}(L+1)^{-\frac{2}{3}}\left(\left(\frac{1}{\sigma_{1:t}+\lambda_{0:t}}+\frac{1}{T}\right)^{-\frac{1}{3}}-\left(\frac{1}{\sigma_{1:t-1}+\lambda_{0:t-1}}+\frac{1}{T}\right)^{-\frac{1}{3}}\right) \\
&\leq d^{\frac{4}{3}}(L+1)^{-\frac{2}{3}}\left(\left({\sigma_{1:t}+\lambda_{0:t}}\right)^{\frac{1}{3}}-\left(\sigma_{1:t-1}+\lambda_{0:t-1}\right)^{\frac{1}{3}}\right) \\
&\leq d^{\frac{4}{3}}(L+1)^{-\frac{2}{3}}(\sigma_{t}+\lambda_t)^{\frac{1}{3}}.
\end{align*}
Therefore, $\eta_{t+1}^{-1}-\eta_t^{-1}\leq C(\sigma_t+\lambda_t)^{p}$ with $C=d^{\frac{4}{3}}(L+1)^{-\frac{2}{3}}$ and $p=\frac{1}{3}$. Also, because of~\pref{lem: upper bound sigma}, choosing $\gamma=4L$ ensures that $\sigma_t\leq \gamma$ for all $t\in[T]$. Moreover, because of the choice of $\lambda_0$ and $T\geq \lambda_0$, we have
\begin{align*}
    \eta_1&=d^{-\frac{4}{3}}(L+1)^{\frac{2}{3}}\left(\frac{1}{\lambda_0}+\frac{1}{T}\right)^{\frac{1}{3}}\\
    &\leq d^{-\frac{4}{3}}(L+1)^{\frac{2}{3}}\lambda_0^{-\frac{1}{3}}\cdot 2^{\frac{1}{3}}\\
    &\leq d^{-\frac{4}{3}}(L+1)^{\frac{2}{3}}\cdot\frac{d^{\frac{1}{3}}}{32}\cdot \frac{1}{16\sqrt{\nu} d^{\frac{1}{3}}(4L+1)^{\frac{1}{3}}+(L+1)^{\frac{2}{3}}} \\
    &=\frac{1}{32d}\cdot\frac{(L+1)^{\frac{2}{3}}}{16\sqrt{\nu} d^{\frac{1}{3}}(4L+1)^{\frac{1}{3}}+(L+1)^{\frac{2}{3}}} \\
    &=\frac{1}{32(d+16\sqrt{\nu} C(4L+1)^p)}.
\end{align*}
Therefore, according to~\pref{lem: stablity 1/2}, $\|\dpy_t-\dpy_{t+1}\|_{\dpH_t}\leq \frac{1}{2}$. In addition, according to~\pref{eq:gradient-norm}, we have $\|\dpg_t\|_{\dpH_t}^*\leq 2d$ for all $t\in [T]$. Therefore,~\pref{eqn: regret-form} holds. Noticing that $\eta_t=d^{-\frac{4}{3}}(L+1)^{\frac{2}{3}}(\sigma_{1:t-1}+\lambda_{0:t-1})^{-\frac{1}{3}}$, and using~\pref{eq:derive-3} and~\pref{eqn: regret-form}, we have
\begin{align*}
    \textsc{Reg Term} &=\E\left[\sum_{t=1}^T \bfh_t(\y_t) - \sum_{t=1}^T \bfh_t(\tilde{\x})\right]\\
    &\leq \E\left[\sum_{t=1}^T \dpell_t(\y_t) - \sum_{t=1}^T  \dpell_t(\tilde{\x})\right] \tag{\pref{eq:derive-3}}\\
    &\leq \frac{\lambda_0}{2} \norm{\tilde{\x}}_2^2 + \frac{\Psi(\tilde{\x})-\Psi(\y_1)}{\eta_{T+1}} + \sum_{t=1}^T 2\eta_t\|\dpg_t\|_{\dpH_t}^{*2}\tag{\pref{eqn: regret-form}}\\
     &\leq \order(1)+ \otil\left(d^{\frac{4}{3}}\nu T^{\frac{1}{3}}\right)+\sum_{t=1}^T\frac{8d^{\frac{2}{3}}(L+1)^{\frac{2}{3}}}{(\sigma_{1:t-1}+\lambda_{0:t-1})^{\frac{1}{3}}},
\end{align*}
where we use the fact that $\eta_{T+1}\geq d^{-\frac{4}{3}}(L+1)^{\frac{2}{3}}T^{-\frac{1}{3}}\geq d^{-\frac{4}{3}}T^{-\frac{1}{3}}$. This finishes the proof.
\end{proof}

Finally, we combine the above terms and show the following theorem, which holds for an arbitrary sequence of $\{\lambda_t\}_{t=1}^T$ with $\lambda_t\in (0,1)$ for all $t\in [T]$, not necessarily satisfying~\pref{eqn:oracle-tuning-Lipschitz-main}.

\begin{theorem}\label{thm: lip-arbitrary}
With any regularization coefficients $\{\lambda_t\}_{t=1}^T\in(0,1)$,~\pref{alg:lift-Lipschitz} guarantees:
\begin{align}
    \Reg\leq \otil\left(\sum_{t=1}^T\frac{d^{\frac{2}{3}}(L+1)^{\frac{2}{3}}}{(\sigma_{1:t-1}+\lambda_{0:t-1})^{\frac{1}{3}}}+\lambda_{1:T}\right),
\end{align}
\end{theorem}
if loss functions $\{f_t\}_{t=1}^T$ are all $L$-Lipschitz and $T\geq \rho'$ (a constant defined in~\pref{alg:lift-Lipschitz}).
\begin{proof}
    Combining~\pref{eq:part-1-lip}, \pref{eq:part-2-lip}, \pref{eq:part-3-lip}, \pref{eq:part-4-lip} and~\pref{eqn:reg-lip}, we have
\begin{align}
    \mathbb{E}\left[\sum_{t=1}^T f_t(x_t)- \sum_{t=1}^T f_t(x)\right] &\leq \otil\left(\sum_{t=1}^T\frac{d^{\frac{2}{3}}(L+1)^{\frac{2}{3}}}{(\sigma_{1:t-1}+\lambda_{0:t-1})^{\frac{1}{3}}}+d^{\frac{4}{3}}\nu T^{\frac{1}{3}}+\lambda_{1:T}\right). \nonumber
\end{align}
\end{proof}

Next we show that if we choose the adaptive regularization coefficients as shown in~\pref{eqn:oracle-tuning-Lipschitz-main}, the obtained regret bound is no worse than the one with an optimal tuning of $\{\lambda_t\}_{t=1}^T$.
\begin{lemma}
\label{lemma:H-function-lip-app}
Consider the following objective 
\begin{equation}
    \label{eq:tuning-objective-lip}
    \calB'(\{\lambda_s\}_{s=1}^t) \triangleq \lambda_{1:t}+\sum_{\tau=1}^t\frac{d^{\frac{2}{3}}(L+1)^{\frac{2}{3}}}{(\sigma_{1:\tau}+\lambda_{0:\tau})^{\frac{1}{3}}},
\end{equation}
with $\lambda_0$ defined in \pref{alg:lift-Lipschitz}. Then the sequence $\{\lambda_t\}_{t=1}^T$ attained by solving~\pref{eqn:oracle-tuning-Lipschitz-main} satisfies that for all $t \in [T]$, $\lambda_t\in (0,1)$ and 
\begin{equation}
  \label{eq:near-optimal-H-function-lip}
  \calB'(\{\lambda_s\}_{s=1}^t)\leq 2\min_{\{\lambda_s^*\}_{s=1}^t\geq 0}\calB'(\{\lambda_s^*\}_{s=1}^t).
\end{equation} 
\end{lemma}

\begin{proof}
    First, we show that there exists a coefficient $\lambda_t\in(0,1)$ for all $t\in [T]$ that satisfies the fixed-point problem~\pref{eq:tuning-objective-lip}. Indeed, we have the following two observations:
\begin{itemize}
    \item on one hand, when setting $\lambda_t=0$, the LHS of~\pref{eq:near-optimal-H-function-lip} equals to $0$, while the RHS of~\pref{eq:near-optimal-H-function-lip} is strictly larger than $0$;
    \item on the other hand, when setting $\lambda_t=1$, the LHS of~\pref{eq:near-optimal-H-function-lip} is equal to $1$ but the RHS of~\pref{eq:near-optimal-H-function-lip} is strictly less than 1 due to the choice of $\lambda_0\geq d^2(L+1)^2$.
\end{itemize}
Combining both facts shows that there exists a coefficient $\lambda_t\in(0,1)$ that satisfies~\pref{eq:near-optimal-H-function-lip}.

    We prove this by induction similar to Lemma 3.1 in \citep{NIPS'07:AOGD}. Again, we set $\lambda_0^*=\lambda_0$. Consider the case of $t=1$. If $\wh{\lambda}_1\leq \lambda_1^*$, we have $\calB'(\wh{\lambda}_1)=\wh{\lambda}_1+\frac{d^{\frac{2}{3}}(L+1)^{\frac{2}{3}}}{(\sigma_1+\wh{\lambda}_{0:1})^{\frac{1}{3}}}=2\wh{\lambda}_1 \leq 2\lambda^*_1 \leq 2\calB'(\lambda^*_1)$. Otherwise, we have $\calB'(\wh{\lambda}_1)=\frac{2d^{\frac{2}{3}}(L+1)^{\frac{2}{3}}}{(\sigma_1+\wh{\lambda}_{0:1})^{\frac{1}{3}}}\leq\frac{2d^{\frac{2}{3}}(L+1)^{\frac{2}{3}}}{(\lambda_{0:1}^*+\sigma_1)^{\frac{1}{3}}}\leq 2\calB'(\lambda_1^*)$.

Suppose we have $\calB'(\{\wh{\lambda}_s\}_{s=1}^{t-1})\leq 2\min_{\{\lambda_s'\}_{s=1}^{t-1}\geq 0}\calB'(\{\lambda_s'\}_{s=1}^{t-1})$. With a slight abuse of notation, we set $\{\lambda_s^*\}_{s=1}^t=\argmin_{\{\lambda_s'\}_{s=1}^t\geq 0}\calB'(\{\lambda_s'\}_{s=1}^t)$. Similarly, if $\wh{\lambda}_{1:t}\leq \lambda_{1:t}^*$, we have
\begin{align*}
    \calB'(\{\wh{\lambda}_{s}\}_{s=1}^t\}) &= \wh{\lambda}_{1:t}+\sum_{s=1}^t \frac{d^{\frac{2}{3}}(L+1)^{\frac{2}{3}}}{(\sigma_{1:s}+\wh{\lambda}_{0:s})^{\frac{1}{3}}} = \wh{\lambda}_{1:t}+\sum_{s=1}^t \wh{\lambda}_{s}\leq 2\lambda_{1:t}^*\leq 2\calB'(\{\lambda_{s}^*\}_{s=1}^t).
\end{align*}
Otherwise, we have
\begin{align*}
      \wh{\lambda}_t+\frac{d^{\frac{2}{3}}(L+1)^{\frac{2}{3}}}{(\sigma_{1:t}+\wh{\lambda}_{0:t})^{\frac{1}{3}}} = \frac{2d^{\frac{2}{3}}(L+1)^{\frac{2}{3}}}{(\sigma_{1:t}+\wh{\lambda}_{0:t})^{\frac{1}{3}}}\leq \frac{2d^{\frac{2}{3}}(L+1)^{\frac{2}{3}}}{(\sigma_{1:t}+{\lambda}_{0:t}^*)^{\frac{1}{3}}}\leq 2\left(\lambda_t^*+\frac{d^{\frac{2}{3}}(L+1)^{\frac{2}{3}}}{(\sigma_{1:t}+{\lambda}_{1:t}^*)^{\frac{1}{3}}}\right).
\end{align*}
Using the induction hypothesis, we have $\calB'(\{\wh{\lambda}_{s}\}_{s=1}^t\})\leq 2\calB'(\{\lambda_{s}^*\}_{s=1}^t)$.
\end{proof}

Therefore, combining~\pref{thm: lip-arbitrary} and~\pref{lemma:H-function-lip-app} gives the proof of~\pref{thm:main-result-lipschitz}.

\subsection{Proofs for Implications of~\pref{thm:main-result-lipschitz}}
\label{appendix: lip implication proof}
In this subsection, we prove the corollaries presented in~\pref{sec:lipschitz-BCO}. 
~\\
\begin{proof}[of~\pref{cor: only-Lipschitz}]
Since \pref{thm:main-result-lipschitz} holds for any sequence of $\{\lambda_t^*\}_{t=1}^T$, in particular, we choose $\lambda_1^*=\sqrt{(L+1)d}T^{\nicefrac{3}{4}}$ and $\lambda_t^*=0$ for all $t \geq 2$, then we obtain that
\begin{align*}
\Reg \overset{\eqref{eqn: lip regret bound}}{\leq} \otil\left(\sqrt{d(L+1)}T^{\frac{3}{4}}\right),
\end{align*}
which completes the proof.
\end{proof}

\begin{proof}[of~\pref{cor: only-Lipschitz-strongly-cvx}]
Choose $\lambda_t^*=0$ for all $t\geq 1$ and by~\pref{thm:main-result-lipschitz}, we obtain that
\begin{align*}
\Reg \overset{\eqref{eqn: lip regret bound}}{\leq} \otil\left( \sum_{t=1}^T \frac{d^{\frac{2}{3}}(L+1)^{\frac{2}{3}}}{\sigma^{\frac{1}{3}} t^{\frac{1}{3}}}\right) = \otil((L+1)^{\frac{2}{3}}d^{\frac{2}{3}}T^{\frac{2}{3}}\sigma^{-\frac{1}{3}}),
\end{align*}
which completes the proof.
\end{proof}

\begin{proof}[of~\pref{cor: only-lipschitz-worst-best-case}]
In the first environment where there are $M$ rounds such that the loss function is $0$-strongly convex. In order to make the right hand side of~\pref{eqn: lip regret bound} the largest, we have $\sigma_{\tau}=0$ when $\tau\in[M]$ and $\sigma_{\tau}=\sigma$ when $\tau>M$. Set $\lambda_t^*=0$ for all $t\geq 2$, then~\pref{thm:main-result-lipschitz} implies that (omitting the $\otil(d^{\frac{4}{3}}\nu T^{\frac{1}{3}})$ low-order term)
\begin{align*}
    \Reg&\leq \otil\left(\lambda_1^* + \sum_{t=1}^T\frac{d^{\frac{2}{3}}(L+1)^{\frac{2}{3}}}{(\sigma_{1:t}+\lambda_1^*)^{\frac{1}{3}}}\right)\\
    &\leq \otil\left(\lambda_1^*+{\lambda_1^*}^{-\frac{1}{3}}Md^{\frac{2}{3}}(L+1)^{\frac{2}{3}}+d^{\frac{2}{3}}(L+1)^{\frac{2}{3}}\min\left\{\sigma^{-\frac{1}{3}}(T-M)^{\frac{2}{3}},{\lambda_1^*}^{-\frac{1}{3}}(T-M)\right\}\right)\\
    &\leq \otil\left(\sqrt{d(L+1)}M^{\frac{3}{4}}+d^{\frac{2}{3}}(L+1)^{\frac{2}{3}}\min\left\{\sigma^{-\frac{1}{3}}(T-M)^{\frac{2}{3}}, \frac{T-M}{{d^{\frac{1}{6}}}(L+1)^{\frac{1}{6}}M^{\frac{1}{4}}}\right\}\right),
\end{align*}
where the last inequality is by choosing $\lambda_1^*=\sqrt{d(L+1)}M^{\frac{3}{4}}$. This proves the first result.

Consider the second type of environment where the first $T-M$ rounds are $\sigma$-strongly convex functions and the remaining rounds are $0$-strongly convex functions. Still set $\lambda_t^*=0$ for all $t\geq 2$ and we have
\begin{align*}
    \Reg&\leq\otil\left(\lambda_1^*+\sum_{t=1}^T\frac{d^{\frac{2}{3}}(L+1)^{\frac{2}{3}}}{(\sigma_{1:t}+\lambda_1^*)^{\frac{1}{3}}}\right) \\
    &\leq \otil\left(\lambda_1^*+d^{\frac{2}{3}}(L+1)^{\frac{2}{3}}\min\left\{{\lambda_1^*}^{-\frac{1}{3}}(T-M), \sigma^{-\frac{1}{3}}(T-M)^{\frac{2}{3}}\right\}+\frac{d^{\frac{2}{3}}M(L+1)^{\frac{2}{3}}}{(\lambda_1^*+\sigma(T-M))^{\frac{1}{3}}}\right).
\end{align*}
When $T-M=\Theta(T)$, we have
\begin{align*}
    \Reg&\leq \otil\left(\lambda_1^*+d^{\frac{2}{3}}(L+1)^{\frac{2}{3}}\min\left\{{\lambda_1^*}^{-\frac{1}{3}}T, \sigma^{-\frac{1}{3}}T^{\frac{2}{3}}\right\}+\frac{d^{\frac{2}{3}}M(L+1)^{\frac{2}{3}}}{(\lambda_1^*+\sigma T)^{\frac{1}{3}}}\right)\\
    &\leq \otil\left(\lambda_1^*+\frac{d^{\frac{2}{3}}T(L+1)^{\frac{2}{3}}}{(\sigma T)^{\frac{1}{3}}}\right)\\
    &\leq \otil\left(\frac{d^{\frac{2}{3}}T(L+1)^{\frac{2}{3}}}{\sigma^{\frac{1}{3}}(T-M)^{\frac{1}{3}}}\right),
\end{align*}
where the last inequality is by choosing $\lambda_1^*=0$. When $T-M=o(T)$, we have $M=\Theta(T)$. Furthermore, when $\lambda_1\leq \sigma(T-M)$, we have ${\lambda_1^*}^{-\frac{1}{3}}(T-M)\geq \sigma^{-\frac{1}{3}}(T-M)^{\frac{2}{3}}$ and therefore,
\begin{align*}
    \Reg\leq \otil\left(\lambda_1^*+d^{\frac{2}{3}}(L+1)^{\frac{2}{3}}\sigma^{-\frac{1}{3}}T^{\frac{2}{3}}+\frac{d^{\frac{2}{3}}T(L+1)^{\frac{2}{3}}}{(\lambda_1^*+\sigma(T-M))^{\frac{1}{3}}}\right)\leq \otil\left(\frac{d^{\frac{2}{3}}(L+1)^{\frac{2}{3}}T}{\sigma^{\frac{1}{3}}(T-M)^{\frac{1}{3}}}\right),
\end{align*}
where the last inequality is by choosing $\lambda_1^*=0$. When $\lambda_1^* \geq \sigma(T-M)$, we have ${\lambda_1^*}^{-\frac{1}{3}}(T-M)\leq \sigma^{-\frac{1}{3}}(T-M)^{\frac{2}{3}}$ and therefore,
\begin{align*}
    \Reg&\leq \otil\left(\lambda_1^*+d^{\frac{2}{3}}{\lambda_1^*}^{-\frac{1}{3}}(L+1)^{\frac{2}{3}}T+\frac{d^{\frac{2}{3}}(L+1)^{\frac{2}{3}}T}{(\lambda_1^*+\sigma(T-M))^{\frac{1}{3}}}\right)\leq \otil\left(\lambda_1^*+d^{\frac{2}{3}}(L+1)^{\frac{2}{3}}{\lambda_1^*}^{-\frac{1}{3}}T\right)\\
    &\leq \otil\left(\sigma(T-M)+\sqrt{d(L+1)}T^{\frac{3}{4}}\right),
\end{align*}
where the last inequality is by choosing $\lambda_1^*=\max\left\{\sigma(T-M), \sqrt{d(L+1)}T^{\frac{3}{4}}\right\}$. Combining the two cases, we have
\begin{align*}
    \Reg\leq \otil\left(\min\left\{\frac{d^{\frac{2}{3}}T(L+1)^{\frac{2}{3}}}{\sigma^{\frac{1}{3}}(T-M)^{\frac{1}{3}}}, \sigma(T-M)+\sqrt{d(L+1)}T^{\frac{3}{4}}\right\}\right),
\end{align*}
leading to the second conclusion.
\end{proof}

\begin{proof}[of~\pref{cor: only-lipschitz-intermediate-case}]
Since \pref{thm:main-result-lipschitz} holds for any sequence of $\{\lambda_t^*\}_{t=1}^T$, in particular, we choose $\lambda_t^*=0$ for all $t\geq 2$ and set $\lambda_1^*=(L+1)^{\mu_0}d^{\mu_1}T^{\mu_2}$, we obtain that (again omitting the low-order term)
\begin{align*}
    \Reg &\leq \Ot\left( \lambda_1^*+\sum_{t=1}^T\frac{d^{\frac{2}{3}}(L+1)^{\frac{2}{3}}}{(t^{1-\alpha}+\lambda_1^*)^{\frac{1}{3}}}\right) \\
    &\leq \otil\left((L+1)^{\mu_0}d^{\mu_1}T^{\mu_2}  + d^{\frac{2}{3}}(L+1)^{\frac{2}{3}}\min\big\{(L+1)^{-\frac{\mu_0}{3}}d^{-\frac{\mu_1}{3}}T^{1-\frac{\mu_2}{3}}, T^{\frac{2+\alpha}{3}}\big\}\right).
\end{align*}
First, the above bound can be upper bounded by
\begin{align*}
    \Reg \leq \otil\left((L+1)^{\mu_0}d^{\mu_1}T^{\mu_2}  + d^{\frac{2-\mu_1}{3}}(L+1)^{\frac{2-\mu_0}{3}}T^{1-\frac{\mu_2}{3}}\right)= \otil\left(\sqrt{d(L+1)}T^{\frac{3}{4}}\right),
\end{align*}
where the last equality is true by choosing $\mu_0=\frac{1}{2}$, $\mu_1=\frac{1}{2}$ and $\mu_2=\frac{3}{4}$. Second, when $\alpha\in [0, \frac{1}{4}-\frac{1}{2}\log_T(L+1)-\frac{1}{2}\log_Td]$, we choose $\mu_0=\mu_1=\mu_2=0$ and have
\begin{align*}
    \Reg &\leq \otil\left(d^{\frac{2}{3}}(L+1)^{\frac{2}{3}}T^{\frac{2+\alpha}{3}}\right).
\end{align*}
Again, we emphasize that the setting of $\lambda_1^*$ is required in the analysis only and will not affect the algorithmic procedures. Combining both situations finishes the proof.
\end{proof}
\section{Self-concordant Barrier Properties}
\label{appendix: self-concordant}

One of the important technical tools used in this paper is the \emph{self-concordant barrier}, which is widely used in the interior-point method and becomes a central concept in modern convex optimization~\citep{nesterov1994-IPM}. The method is introduced to online learning community in the seminal paper of~\citet{Competing:Dark} and successfully resolve several important open problems. Below, we list several basic definitions and some important properties, most of which can be found in~\citep[Section 2]{nemirovski2008-IPM}.

\begin{definition}[Self-Concordant Functions]
\label{def:SC}
Let $\X \subseteq \R^d$ be a closed convex domain with a nonempty interior $\interior(\X)$. A function $\Rcal: \interior(\X) \mapsto \R$ is called \emph{self-concordant} on $\X$ if
\begin{itemize}
  \item[(i)] $\Rcal$ is a three times continuously differentiable convex function, and approaches infinity along any sequence of points approaching $\partial \X$; and 
  \item[(ii)] $\Rcal$ satisfies the differential inequality: for every $\h \in \R^d$ and $x \in \interior(\X)$,
  $$ \abs{\nabla^3 \Rcal(x)[\h,\h,\h]} \leq 2 (\nabla^2 \Rcal(x)[\h,\h])^{\frac{3}{2}},$$
  where the third-order differential is defined as 
  $$\nabla^3 \Rcal(x)[\h, \h, \h] \triangleq \frac{\partial^{3}}{\partial t_{1} \partial t_{2} \partial t_{3}} \Rcal\left(x+t_{1} \h+t_{2} \h+t_{3} \h\right)\Big|_{t_{1}=t_{2}=t_{3}=0}.$$
\end{itemize}
Given a real $\nu \geq 1$, $\Rcal$ is called a $\nu$-\emph{self-concordant barrier} ($\nu$-SCB) for $\X$ if $\Rcal$ is self-concordant on $\X$ and, in addition, for every $\h \in \R^d$ and $x \in \interior(\X)$,
$$ \abs{\nabla\Rcal(x)[\h]} \leq \nu^{\frac{1}{2}} (\nabla^2 \Rcal(x)[\h,\h])^{\frac{1}{2}}.$$
\end{definition}

Given a self-concordant function $\Rcal$ on $\X$, for any $\h \in \R^d$ the induced local norm is defined as 
\begin{equation}
  \label{eq:local-norm}
  \norm{\h}_{x} \triangleq \norm{\h}_{\nabla^2 \Rcal(x)} = \sqrt{\h^{\T} \nabla^2 \Rcal(x) \h},~~\mbox{and}~~\norm{\h}_{x}^{*} \triangleq \norm{\h}_{\nabla^2 \Rcal(x)}^* = \sqrt{\h^{\T} (\nabla^2 \Rcal(x))^{-1} \h}.
\end{equation}
We then introduce the notion of \emph{Dikin ellipsoid} which enjoys nice properties as shown below.
\begin{lemma}
For any closed convex set $\X \subseteq \R^d$ with nonempty interior points, let $\Rcal$ be a self-concordant function on the closed convex set. Then, for any $x \in \interior(\X)$, we have $\Ecal_1(x) \subset \X$, where $\Ecal_1(x)$ denotes the unit Dikin ellipsoid of $\Rcal$ defined as $\Ecal_1(x) \triangleq \{ y \in \R^d \mid \norm{y - x}_{x} \leq 1\}$.
\end{lemma}

\begin{lemma}
\label{lemma:bound-Minkowski}
Let $\Rcal: \interior(\X) \mapsto \R$ be a $\nu$-self-concordant barrier over the closed convex set $\X \in \R^d$, then for any $x, y \in \interior(\X)$, we have $\Rcal(y) - \Rcal(x) \leq \nu \log \frac{1}{1 - \pi_{x}(y)}$, where $\pi_{x}(y) \triangleq \inf\{t \geq 0 \mid x + t^{-1}(y - x) \in \X\}$ is the Minkowski function of $\X$ whose pole is on $x$, which is always in $[0,1]$.
\end{lemma}

Below, we present several key technical lemmas regarding to the self-concordant functions.
\begin{lemma}[{Theorem 2.1.1 of~\cite{nesterov1994-IPM}}]
\label{lemma:shift-norm}
Let $\psi$ be a self-concordant function on the closed convex set $\X \subseteq \R^d$, then
\begin{equation}
    \label{eq:shift-norm}
    \norm{h}_{\nabla^2 \psi(x')} \geq \norm{h}_{\nabla^2 \psi(x)} (1 - \norm{x - x'}_{\nabla^2 \psi(x)})
\end{equation}
holds for any $h \in \R^d$ and any $x \in \interior(\X)$ with $x' \in \Ecal_1(x) \triangleq \{ y \in \R^d \mid \norm{y - x}_{x} \leq 1\}$.
\end{lemma}

\begin{lemma}[{Theorem 2.5.1 of~\cite{nesterov1994-IPM}}]
\label{lemma:exist-SCB}
For each each closed convex domain $\X \subseteq \R^d$, there exits an $\O(d)$-self-concordant barrier on $\X$.
\end{lemma}

\begin{lemma}[{Proposition 5.1.4 of~\cite{nesterov1994-IPM}}]
\label{lemma:normal-barrier-nu}
Suppose $\psi$ is a $\nu$-self-concordant barrier on $\X \subseteq \R^d$. Then the function 
\[
    \Psi(w,b) \triangleq 400 \left(\psi\big(\frac{w}{b}\big) - 2\theta \ln b\right)
\]
is a $\bar{\nu}$-self-concordant barrier on $con(\X) \subseteq \R^{d+1}$ with $\bar{\nu} = 800 \nu$, where $con(\X) = \{ \boldsymbol{0}\} \cup \{ (w,b) \mid \frac{w}{b} \in \X, w \in \R^d, b>0\}$ is the conic hull of $\X$ lifted to $\R^{d+1}$ (by appending a dummy variable $1$ to the last coordinate).
\end{lemma}

\begin{lemma}[{Proposition 2.3.4 of~\cite{nesterov1994-IPM}}]
\label{lemma:normal-barrier}
Suppose $\psi$ is a $\nu$-normal barrier on $\X \subseteq \R^d$. Then for any $x,y\in \interior(\X)$, we have
\begin{enumerate}
    \item[(1)] $\norm{x}_{\nabla^2 \psi(x)}^2 = x^\T \nabla^2 \psi(x) x = \nu$;
    \item[(2)] $\nabla^2 \psi(x) x = - \nabla \psi(x)$;
    \item[(3)] $\psi(y) \geq \psi(x) - \nu \ln \frac{-\inner{\nabla \psi(x),y}}{\nu}$.
    \item[(4)] $\|\nabla\psi(x)\|_{\nabla^{-2}\psi(x)}^2=\nu$.
\end{enumerate}
\end{lemma}
\section{Additional Lemmas}\label{sec: tech_lemma}

\subsection{FTRL Lemma}\label{sec:technical-lemmas}

For completeness, we present the following general result for FTRL-type algorithms as follows.
\begin{lemma}
\label{lemma:FTRL-regret}
Let $\X \subseteq \R^d$ be a closed and convex feasible set, and denote by $R_t: \X \mapsto \R$ the convex regularizer and by $f_t: \X \mapsto \R$ the convex online functions. Denote by $F_t(x) = R_t(x) + \sum_{s=1}^{t-1} f_s(x)$ and the FTRL update rule is specified as $x_t \in \argmin_{x \in \X} F_t(x)$. Then, for any $u \in \X$ we have
\begin{equation}
    \label{eq:FTRL-regret}
    \begin{split}    
    \sum_{t=1}^T f_t(x_t) - \sum_{t=1}^T f_t(u) \leq  & R_{T+1}(u) - R_1(x_1) + \sum_{t=1}^T \nabla f_t(x_t)^\T (x_t - x_{t+1}) \\
     & - \sum_{t=1}^T D_{F_t + f_t}(x_{t+1},x_t) + \sum_{t=1}^T \Big( R_t(x_{t+1})-R_{t+1}(x_{t+1})\Big),
    \end{split}
\end{equation}
where $D_{F_t + f_t}(\cdot,\cdot)$ denotes the Bregman divergence induced by the function $F_t + f_t$.
\end{lemma}
\begin{proof}
It is easy to verify that the following equation holds for any comparator $u \in \X$,
\[
    \begin{split}
    \sum_{t=1}^T f_t(x_t) - \sum_{t=1}^T f_t(u) =  & R_{T+1}(u) - R_1(x_1) + F_{T+1}(x_{T+1}) - F_{T+1}(u) \\
     & \qquad +  \sum_{t=1}^T \Big(F_t(x_t) - F_{t+1}(x_{t+1}) + f_t(x_t)\Big).
    \end{split}
\]
Moreover, we have
\begin{align*}
    & F_t(x_t) - F_{t+1}(x_{t+1}) + f_t(x_t)\\
    & = F_t(x_t) + f_t(x_t) - \left(F_t(x_{t+1}) + f_t(x_{t+1})\right) + R_t(x_{t+1}) - R_{t+1}(x_{t+1})\\
    & = \inner{\nabla F_t(x_t) + \nabla f_t(x_t),x_t - x_{t-1}} - D_{F_t + f_t}(x_{t+1},x_t) + R_t(x_{t+1})-R_{t+1}(x_{t+1})\\
    & \leq \inner{\nabla f_t(x_t),x_t - x_{t-1}} - D_{F_t + f_t}(x_{t+1},x_t) + R_t(x_{t+1})-R_{t+1}(x_{t+1})
\end{align*}
where the last inequality holds by the optimality condition of $x_t \in \argmin_{x\in \X} F_t(x)$. Hence, combining the above equations finishes the proof.
\end{proof}

\subsection{Relations among strong convexity, smoothness and Lipschitzness}

In this section, we discuss the relations among strong convexity, smoothness and Lipschitzness. First, we point out a minor technical flaw that appeared in two previous works on BCO \citep{AISTATS'11:smooth-BCO,conf/nips/HazanL14}. In both works, the authors use the statement that a convex function $f$ that is $\beta$-smooth and has bounded value in $[-1,1]$ has Lipschitz constant no more than $2\beta+1$ when $\max_{x,x'\in\calX}\|x-x'\|_2\in [2,4]$. However, this is not correct as we give the following counter example.
\begin{example}
    Consider the following function in $2$-dimensional space: $f(x,y)=Gy$ where $G>1$ can be arbitrarily large and the first coordinate does not affect the function value. The feasible domain is defined as $\calX=\{(x,y)\;|\;x\in[-1,1], y\in [-\frac{1}{G},\frac{1}{G}]\}$ with diameter in $[2,4]$. It is direct to see that function $f$ is $0$-smooth and has bounded value in $[-1,1]$. However the Lipschitz constant is $G$, which can be arbitrarily large.
\end{example}
\citet{conf/nips/HazanL14} and~\citet{AISTATS'11:smooth-BCO} use this property to bound the term~\textsc{Comparator Bias} in \pref{eq:decomposition}. We fix that by using the property of convexity.

Next, we discuss the relationship between strong convexity and Lipschitzness. Specifically, the following lemma shows that for a convex function $f$ that is $L$-Lipschitz and defined over a bounded domain with diameter $D$, its strong convexity parameter $\sigma$ is upper bounded by $\frac{4L}{D}$.
\begin{lemma}\label{lem: upper bound sigma}
If a convex function $f:\calX\mapsto\mathbb{R}$ is $L$-Lipschitz and $\sigma$-strongly convex, and has bounded domain diameter $\max_{x,x'\in \calX}\|x-x'\|_2= D$, then we have $\sigma\leq\frac{4L}{D}$.
\end{lemma}
In fact, we have for any $x,y\in \calX$
\begin{align*}
    L\|x-y\|_2 \geq f_t(x) - f_t(y) \geq \nabla f_t(y)^\top (x-y) + \frac{\sigma}{2}\|x-y\|_2^2.
\end{align*}
Choose $y=\argmin_{x\in\calX}f_t(x)$ and we have
\begin{align*}
    L\|x-y\|_2 \geq f_t(x) - f_t(y) \geq \nabla f_t(y)^\top (x-y) + \frac{\sigma}{2}\|x-y\|_2^2\geq \frac{\sigma}{2}\|x-y\|_2^2.
\end{align*}
Therefore, $\sigma\leq \frac{2L}{\|x-y\|_2}$ for any $x\in \calX$, which means that $\sigma\leq \frac{4L}{D}$. This is because we can choose $x_1,x_2\in\calX$ such that $\|x_1-x_2\|=D=\max_{x,x'\in\calX}\|x-x'\|_2$. Then we have $\|x_1-y\|_2+\|x_2-y\|_2\geq \|x_1-x_2\|_2=D$, which means that either $\|x_1-y\|\geq \frac{D}{2}$ or $\|x_2-y\|_2\geq \frac{D}{2}$. This shows that when $f_t$ is both $\sigma$-strongly convex and $L$-Lipschitz, we have $\sigma\leq\frac{4L}{D}$.

\end{document}